\documentclass{article}

\usepackage{etoolbox}
\newtoggle{TODO}
\newtoggle{appendix}
\newtoggle{nips}
\togglefalse{TODO}
\toggletrue{appendix}
\togglefalse{nips}

\iftoggle{nips}{
  \PassOptionsToPackage{numbers,sort}{natbib}
  \usepackage[final]{nips_2016}
}{

  \setlength{\textwidth}{6.5in}
  \setlength{\textheight}{9in}
  \setlength{\oddsidemargin}{0in}
  \setlength{\evensidemargin}{0in}
  \setlength{\topmargin}{-0.5in}
  \newlength{\defbaselineskip}
  \setlength{\defbaselineskip}{\baselineskip}
  \setlength{\marginparwidth}{0.8in}
  \usepackage[numbers,sort]{natbib}
  \date{}
}

\newcommand{\mailto}[1]{\href{mailto:#1@stanford.edu}{#1}}

\title{Scan Order in Gibbs Sampling: Models in Which it Matters and Bounds on How Much}

\author{
  Bryan He,
  Christopher De Sa,
  Ioannis Mitliagkas, and
  Christopher R{\'e} \\
  Stanford University \\
  \texttt{\{\mailto{bryanhe},\mailto{cdesa},\mailto{imit},\mailto{chrismre}\}@stanford.edu}
}

\usepackage[utf8]{inputenc} 
\usepackage[T1]{fontenc}    
\makeatletter
\usepackage[pdftitle={\@title},
            pdfauthor={Bryan He, Christopher De Sa, Ioannis Mitliagkas, Christopher R{\'e}},
            pdfsubject={},
            pdfkeywords={Gibbs sampling, scan order}]{hyperref}       
\usepackage{bookmark}
\makeatother
\usepackage{url}            
\usepackage{booktabs}       
\usepackage{amsfonts}       
\usepackage{nicefrac}       
\usepackage{microtype}      

\usepackage{amsopn,amsmath,amsthm,amssymb}
\usepackage{algorithmic}
\usepackage{algorithm}
\usepackage{multirow}
\usepackage{pgfplots}
\usepackage{tikz}
\iftoggle{nips}{
  \usetikzlibrary{calc,math}
  \usepgflibrary{arrows.meta,decorations.markings,decorations.pathreplacing}
  \pgfplotsset{compat=1.11}
}{}
\usepackage{subcaption}
\usepackage{caption} 
\usepackage{thmtools}
\usepackage{versions}

\renewcommand{\arraystretch}{1.1}

\theoremstyle{definition}

\theoremstyle{plain}

\newtheorem{claim}{Claim}

\newcommand{\todo}[1]{\iftoggle{TODO}{\textcolor{red}{#1}}}

\begin{document}

\maketitle

\begin{abstract}
Gibbs sampling is a Markov Chain Monte Carlo sampling technique that
iteratively samples variables from their conditional distributions.
There are two common scan orders for the variables:
random scan and systematic scan.
Due to the benefits of locality in hardware, systematic scan is commonly used,
even though most statistical guarantees are only for random scan.
While it has been conjectured that the mixing times of random scan and
systematic scan do not differ by more than a logarithmic factor, we show by
counterexample that this is not the case, and we prove that
that the mixing times do not differ by more than a polynomial factor
under mild conditions.
To prove these relative bounds, we introduce a method of augmenting the state space to
study systematic scan using conductance.
\end{abstract}


\section{Introduction}
\label{sec:intro}

Gibbs sampling, or Glauber dynamics, is a Markov chain Monte Carlo
method that draws approximate samples from multivariate distributions
that are difficult to sample directly~\mbox{[\citealp{geman1984paml}; \citealp[p.~40]{levin2009ams}]}.
A major use of Gibbs sampling is marginal inference: the estimation of the marginal distributions of some variables of interest \cite{gelfand1990jasa}.
Some applications include 
various computer vision tasks~\cite{zhu1998filters,geman1984paml,zhang2001segmentation}, information extraction \cite{finkel2005acl}, and latent Dirichlet allocation for topic modeling \cite{griffiths2004pnas}.
Gibbs sampling is simple to implement and quickly produces accurate samples for many models, so it is widely used and available in popular libraries such as OpenBUGS \cite{lunn2009bugs}, FACTORIE \cite{mccallum2009nips}, JAGS \cite{plummer2003jags}, and MADlib~\cite{hellerstein2012vldb}.

\begin{algorithm}[htbp]
  \caption{Gibbs sampler}
  \label{alg:Gibbs}
  \begin{algorithmic}
    \INPUT Variables $x_i$ for $1 \leq i \leq n$, and target distribution $\pi$
    \STATE Initialize $x_1,\ldots,x_n$
    \LOOP
      \STATE Select variable index $s$ from $\{1,\ldots,n\}$
      \STATE Sample $x_s$ from the conditional distribution $\mathbf{P}_\pi\left(X_s\mid X_{\{1,\ldots,n\}\setminus\{s\}}\right)$
    \ENDLOOP
  \end{algorithmic}
\end{algorithm}

Gibbs sampling (Algorithm \ref{alg:Gibbs}) iteratively selects
a single variable and resamples it from its conditional distribution, given
the other variables in the model.
The method that selects the variable index to sample
($s$ in Algorithm~\ref{alg:Gibbs}) is called the \textit{scan order}.
Two scan orders are commonly used: random scan and systematic scan (also known as deterministic or sequential scan).
In random scan, the variable to sample is selected uniformly and independently
at random at each iteration.
In systematic scan, a fixed permutation is selected, and the variables are
repeatedly selected in that order.
The existence of these two distinct options raises an obvious question---which scan order produces accurate samples more quickly?
This question has two components: hardware efficiency (how long does each
iteration take?) and statistical efficiency (how many iterations are needed
to produce an accurate sample?).

From the hardware efficiency perspective, systematic scans are clearly superior \cite{zhang2013sigmod,smola2010vldb}.
Systematic scans have good spatial locality because they access the variables in linear order, which makes their iterations run faster on hardware.
As a result, systematic scans are commonly used in practice.

Comparing the two scan orders is much more interesting from the perspective of
statistical efficiency, which we focus on for the rest of this paper.
Statistical efficiency is measured by the \textit{mixing time}, which is the
number of iterations needed to obtain an accurate
sample~\cite[p.~55]{levin2009ams}.
The mixing times of random scan and systematic scan have been studied,
and there is a longstanding
conjecture~\mbox{[\citealp{diaconis2013bernoulli}; \citealp[p.~300]{levin2009ams}]}
that systematic scan (1) never mixes more than a constant factor slower
than random scan and (2) never mixes more than a logarithmic factor
faster than random scan.  This conjecture implies that the choice of scan
order does not have a large effect on performance.

Recently, \citet{roberts2015jsp}
described a model in which systematic scan mixes more slowly
than random scan by a polynomial factor; this disproves direction (1) of this
conjecture.  Independently, we constructed
other models for which the scan order has a significant effect on mixing time.
This raises the question: 
what are the true bounds on the difference between these mixing times?
In this paper, we address this question and make the following contributions.
\begin{itemize}
  \item In Section \ref{sec:examples}, we study the effect of the variable permutation chosen for systematic scan on the mixing time.
    In particular,
    in Section \ref{ssSequenceDependencies}, we construct a model
    for which a systematic scan mixes a
    polynomial factor faster than random scan, disproving direction (2)
    of the conjecture, and
    in Section \ref{ssTwoIslands},
    we construct a model for which the systematic scan with the worst-case
    permutation results in a mixing time that is slower by a polynomial
    factor than both the best-case systematic scan permutation and random scan.
  \item In Section \ref{sec:experiments}, we empirically verify the mixing times of the models we construct,
    and we analyze how the mixing time changes as a function of
    the permutation.
  \item In Section \ref{sec:bounds}, we prove a weaker version of the
    conjecture described above,
    providing relative bounds on the mixing times of random and systematic
    scan.  Specifically, under a mild condition, different scan orders can
    only change the mixing time by a polynomial factor.
    To obtain these bounds, we introduce a method of augmenting the state space
    of
    Gibbs sampling so that the method of conductance can be applied to analyze
    its dynamics.
\end{itemize}

\section{Related Work}

Recent work has made progress on analyzing the mixing time of Gibbs sampling,
but there are still some major limitations to our understanding.
In particular, most known results are only for specific models or are only for random scan.
For example, mixing times are known for Mallow's
model \cite{benjamini2005ams,diaconis2000michigan}, and colorings of a
graph \cite{dyer2006aap} for both random and systematic scan, but these 
are not applicable to general models.
On the other hand, random scan has been shown to mix in polynomial time for
models that satisfy structural conditions -- such as having close-to-modular
energy functions \cite{gotovos2015nips} or having bounded hierarchy width and
factor weights \cite{desa2015nips} -- but corresponding results for for
systematic scan are not known.
The major exception to these limitations is Dobrushin's condition, which
guarantees $O(n\log n)$ mixing for both random scan and systematic
scan \cite{hayes2006focs,dyer2008cpc}.
However, there are many models of interest with close-to-modular energy
functions or bounded hierarchy width that do not satisfy Dobrushin's
condition.



A similar choice of scan order appears in stochastic gradient descent (SGD), where the standard SGD algorithm uses random scan, and the incremental gradient method (IGM) uses systematic scan.
In contrast to Gibbs sampling, avoiding ``bad permutations'' in the IGM is known to be important to ensure fast convergence \cite{recht2012colt,gurbuzbalaban2015arxiva}. 
In this paper, we bring some intuition about the existence of bad permutations from SGD to Gibbs sampling.

\section{Models in Which Scan Order Matters}
\label{sec:examples}

Despite a lack of theoretical results regarding the effect of scan order on
mixing times, it is generally believed that scan order only has a small effect
on mixing time.
In this section, we first define relevant terms and 
state some common conjectures regarding scan order.
Afterwards, we give several counterexamples showing that the scan order can
have asymptotic effects on the mixing time.

The \textit{total variation distance} between two probability distributions $\mu$ and $\nu$ on $\Omega$ is \cite[p.~47]{levin2009ams}
\begin{align*}
\|\mu-\nu\|_{\textrm{TV}} = \max_{A\subseteq \Omega}|\mu(A) - \nu(A)|.
\end{align*}
The \textit{mixing time} is the minimum number of steps needed to guarantee that the total variation distance between the true and estimated distributions is below a given threshold $\epsilon$ from any starting distribution.
Formally, the \textit{mixing time} of a stochastic process $P$ with transition matrix $P^{(t)}$ after $t$ steps and stationary distribution $\pi$ is \cite[p.~55]{levin2009ams}
\begin{align*}
t_{\textrm{mix}}(P,\epsilon) = \min\left\{t : \max_{\mu}\|P^{(t)}\mu - \pi\|_{\textrm{TV}} \leq \epsilon\right\},
\end{align*}
where the maximum is taken over the distribution $\mu$ of the initial state of
the process.  When comparing the statistical efficiency of systematic scan
and random scan, it would be useful to establish, for any systematic scan
process $S$ and random scan process $R$ on the same $n$-variable model,
a relative bound of the form
\begin{align}
  \label{eqnRelBoundSeqVsRand}
  F_1(\epsilon, n, t_{\textrm{mix}}(R,\epsilon))
  \leq
  t_{\textrm{mix}}(S,\epsilon)
  \leq
  F_2(\epsilon, n, t_{\textrm{mix}}(R,\epsilon))
\end{align}
for some functions $F_1$ and $F_2$.  Similarly, to bound the effect
that the choice of permutation can have on the mixing time, it would be
useful to know, for any two systematic scan processes $S_{\alpha}$ and
$S_{\beta}$ with different permutations on the same model, that
for some function $F_3$,
\begin{align}
  \label{eqnRelBoundSeqVsSeq}
  t_{\textrm{mix}}(S_\alpha, \epsilon)
  \leq
  F_3(\epsilon, n, t_{\textrm{mix}}(S_\beta, \epsilon)).
\end{align}

\citet{diaconis2013bernoulli} and \citet[p.~300]{levin2009ams} conjecture that
systematic scan is never more than a constant factor slower or a
logarithmic factor faster than random scan.  This is equivalent to choosing
$F_1(\epsilon, n, t) = C_1(\epsilon) \cdot t \cdot (\log n)^{-1}$ and
$F_2(\epsilon, n, t) = C_2(\epsilon) \cdot t$ in the inequality in
(\ref{eqnRelBoundSeqVsRand}), for some functions $C_1$ and $C_2$.
It is also commonly believed that all systematic scans mix at the same
asymptotic rate, which is equivalent to choosing
$F_3(\epsilon, n, t) = C_3(\epsilon) \cdot t$ in (\ref{eqnRelBoundSeqVsSeq}).

These conjectures imply that using systematic scan instead of random scan
will not result in significant consequences, at least asymptotically,
and that the particular permutation used for systematic scan is not important.
However, we show that neither conjecture is true by constructing models
(listed in Table \ref{table:list}) in which the scan order has substantial
asymptotic effects on mixing time.

In the rest of this section, we go through two models in detail to highlight
the diversity of behaviors that different scan orders can have.
First, as a warm-up, 
we construct a model, which we call the \emph{sequence of
dependencies} model, for which a single ``good permutation'' of systematic scan
mixes faster, by a polynomial factor, than both random scan and systematic
scans using most other permutations.  This serves as a counterexample to the
conjectured lower bounds (i.e. the choice of $F_1$ and $F_3$) on the mixing
time of systematic scan.  Second, we construct a model, the \emph{two islands}
model, in which there is a small set of ``bad permutations'' that
mix very slowly in comparison to random scan and most other systematic scans.
This contradicts the conjectured upper bounds (i.e. the choice of $F_2$
and $F_3$).  For completeness, we also discuss the \emph{discrete pyramid}
model introduced by \citet{roberts2015jsp} (which contradicts the 
conjectured choice of $F_2$).  Table \ref{table:list} also lists the mixing
times of a few additional models we
constructed: these models further explore the space of asymptotic
comparisons among scan orders, but for brevity we defer them to the appendix.

\begin{table}[tb]
\caption{Models and Approximate Mixing Times}
\label{table:list}
\begin{center}
\begin{tabular}{@{}lccc@{}}
\toprule
Model & $t_{\textrm{mix}}(R)$ & $\displaystyle \min_{\alpha}t_{\textrm{mix}}(S_{\alpha})$ & $\displaystyle \max_{\alpha}t_{\textrm{mix}}(S_{\alpha})$ \\
\midrule
Sequence of Dependencies & $n^2$ & $n$ & $n^2$ \\
Two Islands              & $2^n$ & $2^n$ & $n2^n$ \\
Discrete Pyramid \cite{roberts2015jsp}        & $n$ & $n^3$ & $n^3$ \\
Memorize and Repeat      & $n^3$ & $n^2$ & $n^2$ \\
Soft Dependencies        & $n^{3/2}$ & $n$ & $n^2$ \\
\bottomrule
\end{tabular}
\end{center}
\end{table}

\iftoggle{nips}{
  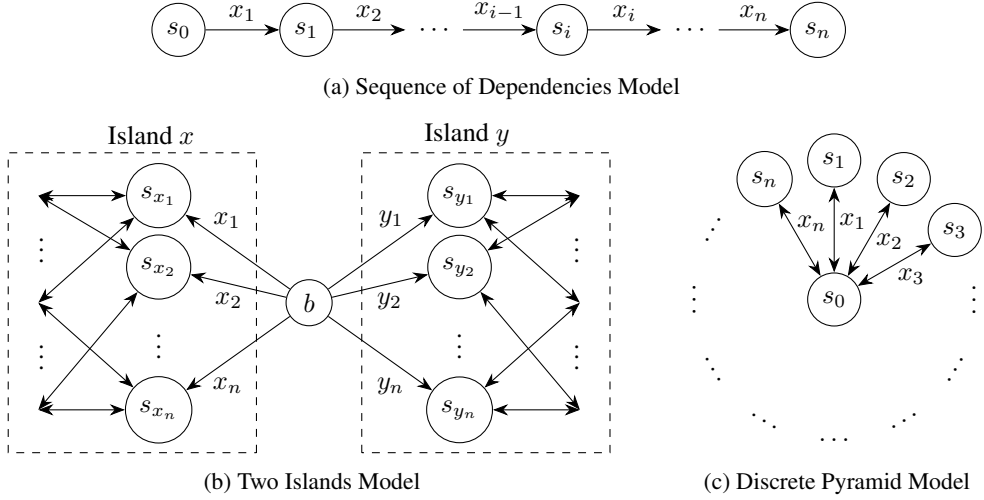
\begin{figure}[tb]
  \begin{center}
  \begin{subfigure}[b]{\textwidth}
  \begin{center}
  \begin{tikzpicture}
  \newcommand*{\size}{0.6cm}
  \newcommand*{\x}{1.7cm}

  \node[circle, minimum size=\size, draw] at ( 0   * \x,  0 * \x) (1){$s_0$};
  \node[circle, minimum size=\size, draw] at ( 1   * \x,  0 * \x) (2){$s_1$};
  \node[circle, minimum size=\size      ] at ( 2   * \x,  0 * \x) (a){$\cdots$};
  \node[circle, minimum size=\size, draw] at ( 3   * \x,  0 * \x) (i){$s_{i}$};
  \node[circle, minimum size=\size      ] at ( 4   * \x,  0 * \x) (b){$\cdots$};
  \node[circle, minimum size=\size, draw] at ( 5   * \x,  0 * \x) (n){$s_n$};

  \draw[black, -{Stealth[scale=1.2]}] (1) -- (2) node[midway,above]{$x_1$};
  \draw[black, -{Stealth[scale=1.2]}] (2) -- (a) node[midway,above]{$x_2$};
  \draw[black, -{Stealth[scale=1.2]}] (a) -- (i) node[midway,above]{$x_{i-1}$};
  \draw[black, -{Stealth[scale=1.2]}] (i) -- (b) node[midway,above]{$x_i$};
  \draw[black, -{Stealth[scale=1.2]}] (b) -- (n) node[midway,above]{$x_n$};
\end{tikzpicture}
  \caption{Sequence of Dependencies Model}
  \label{fig:sequence}
  \end{center}
  \end{subfigure}
  \\
  \vspace{0.12cm}
  \begin{subfigure}[b]{0.625\textwidth} 
  \begin{center}
  \begin{tikzpicture}
  \newcommand*{\size}{0.5cm}
  \newcommand*{\y}{0.95}
  \newcommand*{\x}{2.0}
  \newcommand*{\ang}{30}

  \node[circle, minimum size=\size, draw, ] at (  0,             0) (b)  {$b$};

  \node[circle, minimum size=\size, draw, ] at (-\x,  1.5 * \y) (x1) {$s_{x_1}$};
  \node[circle, minimum size=\size, draw, ] at (-\x,  0.5 * \y) (x2) {$s_{x_2}$};
  \node[        minimum size=\size,       ] at (-\x, -0.5 * \y) (xi) {$\vdots$};
  \node[circle, minimum size=\size, draw, ] at (-\x, -1.5 * \y) (xn) {$s_{x_n}$};

  \node[circle, minimum size=\size, draw, ] at ( \x,  1.5 * \y) (y1) {$s_{y_1}$};
  \node[circle, minimum size=\size, draw, ] at ( \x,  0.5 * \y) (y2) {$s_{y_2}$};
  \node[        minimum size=\size,       ] at ( \x, -0.5 * \y) (yi) {$\vdots$};
  \node[circle, minimum size=\size, draw, ] at ( \x, -1.5 * \y) (yn) {$s_{y_n}$};

  \draw[black, -{Stealth[scale=1.2]}] (b) -- (x1) node[pos=0.6,above=0.1cm]{$x_1$};
  \draw[black, -{Stealth[scale=1.2]}] (b) -- (x2) node[pos=0.6,below]{$x_2$};
  \draw[black, -{Stealth[scale=1.2]}] (b) -- (xn) node[pos=0.6,below=0.1cm]{$x_n$}; 

  \draw[black, -{Stealth[scale=1.2]}] (b) -- (y1) node[pos=0.6,above=0.1cm]{$y_1$};
  \draw[black, -{Stealth[scale=1.2]}] (b) -- (y2) node[pos=0.6,below]{$y_2$};
  \draw[black, -{Stealth[scale=1.2]}] (b) -- (yn) node[pos=0.6,below=0.1cm]{$y_n$};

  \draw[black, {Stealth[scale=1.2]}-{Stealth[scale=1.2]}] (x1) -- (-1.8 * \x,  1.50 * \y);
  \draw[black, {Stealth[scale=1.2]}-{Stealth[scale=1.2]}] (x1) -- (-1.8 * \x,  0.00 * \y);

  \draw[black, {Stealth[scale=1.2]}-{Stealth[scale=1.2]}] (x2) -- (-1.8 * \x,  1.50 * \y);
  \draw[black, {Stealth[scale=1.2]}-{Stealth[scale=1.2]}] (x2) -- (-1.8 * \x, -1.50 * \y);

  \draw[black, {Stealth[scale=1.2]}-{Stealth[scale=1.2]}] (xn) -- (-1.8 * \x, -1.50 * \y);
  \draw[black, {Stealth[scale=1.2]}-{Stealth[scale=1.2]}] (xn) -- (-1.8 * \x,  0.00 * \y);

  \node[] at                                                      (-1.78* \x,  0.85 * \y) {$\vdots$};
  \node[] at                                                      (-1.78* \x, -0.65 * \y) {$\vdots$};

  \draw[black, {Stealth[scale=1.2]}-{Stealth[scale=1.2]}] (y1) -- ( 1.8 * \x,  1.50 * \y);
  \draw[black, {Stealth[scale=1.2]}-{Stealth[scale=1.2]}] (y1) -- ( 1.8 * \x,  0.00 * \y);

  \draw[black, {Stealth[scale=1.2]}-{Stealth[scale=1.2]}] (y2) -- ( 1.8 * \x,  1.50 * \y);
  \draw[black, {Stealth[scale=1.2]}-{Stealth[scale=1.2]}] (y2) -- ( 1.8 * \x, -1.50 * \y);

  \draw[black, {Stealth[scale=1.2]}-{Stealth[scale=1.2]}] (yn) -- ( 1.8 * \x, -1.50 * \y);
  \draw[black, {Stealth[scale=1.2]}-{Stealth[scale=1.2]}] (yn) -- ( 1.8 * \x,  0.00 * \y);

  \node[] at                                                      ( 1.78* \x,  0.85 * \y) {$\vdots$};
  \node[] at                                                      ( 1.78* \x, -0.65 * \y) {$\vdots$};

  \draw[dashed] ( 0.35 * \x, -2.1 * \y) -- ( 0.35 * \x, 2.1 * \y) -- ( 2.0 * \x, 2.1 * \y) -- ( 2.0 * \x, -2.1 * \y) -- cycle;
  \draw[dashed] (-0.35 * \x, -2.1 * \y) -- (-0.35 * \x, 2.1 * \y) -- (-2.0 * \x, 2.1 * \y) -- (-2.0 * \x, -2.1 * \y) -- cycle;

  \node[] at (-1.05 * \x, 2.35 * \y) {Island $x$};
  \node[] at ( 1.05 * \x, 2.35 * \y) {Island $y$};

\end{tikzpicture}
  \caption{Two Islands Model}
  \label{fig:island}
  \end{center}
  \end{subfigure}
  \begin{subfigure}[b]{0.345\textwidth} 
  \begin{center}
  \begin{tikzpicture}
  \newcommand*{\size}{0.0cm}
  \newcommand*{\x}{1.85}
  \newcommand*{\ang}{30}

  \pgfmathsetmacro\ax{sin( 0 * \ang) * \x}
  \pgfmathsetmacro\ay{cos( 0 * \ang) * \x}
  \pgfmathsetmacro\bx{sin( 1 * \ang) * \x}
  \pgfmathsetmacro\by{cos( 1 * \ang) * \x}
  \pgfmathsetmacro\cx{sin( 2 * \ang) * \x}
  \pgfmathsetmacro\cy{cos( 2 * \ang) * \x}
  \pgfmathsetmacro\zx{sin(-1 * \ang) * \x}
  \pgfmathsetmacro\zy{cos(-1 * \ang) * \x}

  \node[circle, minimum size=\size, draw, ] at (  0,   0) (0) {$s_0$};
  \node[circle, minimum size=\size, draw, ] at (\ax, \ay) (1) {$s_1$};
  \node[circle, minimum size=\size, draw, ] at (\bx, \by) (2) {$s_2$};
  \node[circle, minimum size=\size, draw, ] at (\cx, \cy) (3) {$s_3$};
  \foreach \i in {3,...,10}
  {
    \pgfmathsetmacro\iang{\i * \ang}
    \pgfmathsetmacro\ix{sin(\i * \ang) * \x}
    \pgfmathsetmacro\iy{cos(\i * \ang) * \x}
    \node at (\ix, \iy) {\rotatebox{-\iang}{$\cdots$}};
  }
  \node[circle, minimum size=\size, draw, ] at (\zx, \zy) (n){$s_n$};

  \draw[black, {Stealth[scale=1.2]}-{Stealth[scale=1.2]}] (0) -- (1) node[pos=0.60,xshift=0.2500cm,yshift= 0.0000cm]{$x_1$};
  \draw[black, {Stealth[scale=1.2]}-{Stealth[scale=1.2]}] (0) -- (2) node[pos=0.60,xshift=0.2165cm,yshift=-0.1250cm]{$x_2$};
  \draw[black, {Stealth[scale=1.2]}-{Stealth[scale=1.2]}] (0) -- (3) node[pos=0.60,xshift=0.1250cm,yshift=-0.2165cm]{$x_3$};
  \draw[black, {Stealth[scale=1.2]}-{Stealth[scale=1.2]}] (0) -- (n) node[pos=0.60,xshift=0.2165cm,yshift= 0.1250cm]{$x_n$}; 
\end{tikzpicture}
  \caption{Discrete Pyramid Model}
  \label{fig:pyramid}
  \end{center}
  \end{subfigure}
  \caption{State space of the models.}
  \end{center}
  \end{figure}
}{
  \begin{figure}[tb]
  \begin{center}
  \begin{subfigure}[b]{\textwidth}
  \begin{center}
  \includegraphics{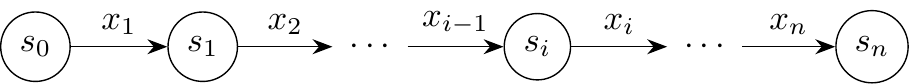}
  \caption{Sequence of Dependencies Model}
  \label{fig:sequence}
  \end{center}
  \end{subfigure}
  \\
  \vspace{0.12cm}
  \begin{subfigure}[b]{0.535\textwidth} 
  \begin{center}
  \includegraphics{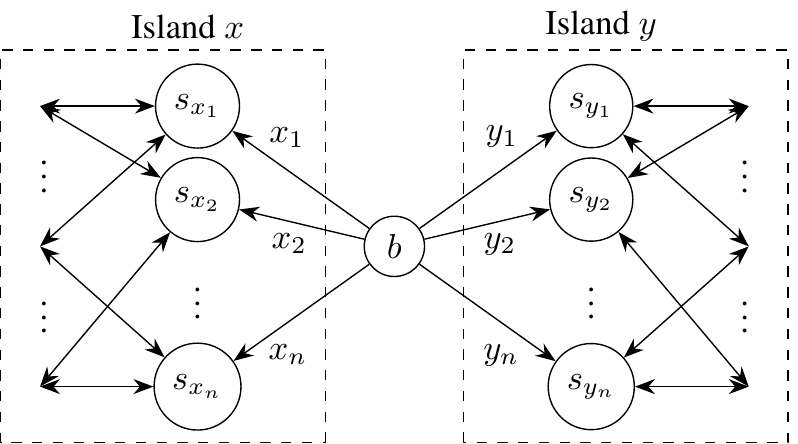}
  \caption{Two Islands Model}
  \label{fig:island}
  \end{center}
  \end{subfigure}
  \begin{subfigure}[b]{0.300\textwidth} 
  \begin{center}
  \includegraphics{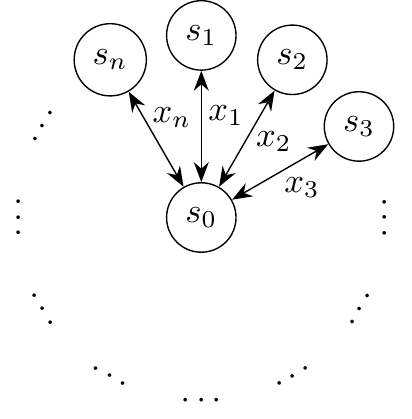}
  \caption{Discrete Pyramid Model}
  \label{fig:pyramid}
  \end{center}
  \end{subfigure}
  \caption{State space of the models.}
  \end{center}
  \end{figure}
}

\subsection{Sequence of Dependencies}
\label{ssSequenceDependencies}

The first model we will describe is the sequence of dependencies
model (Figure \ref{fig:sequence}).
The goal of this model is to explore the question of how fast
systematic scan can be, by constructing a model for which systematic scan
mixes rapidly for one particular good permutation.  The sequence of
dependencies
model achieves this by having the property that, at any time, progress
towards mixing is only made if a particular variable is sampled; this variable
is always the one that is chosen by the good permutation.  As a result,
while a systematic scan using a good permutation makes progress at every step,
both random scan and other systematic scans often fail to make progress, which
leads to a gap between their mixing times.
Thus, this model exhibits two surprising behaviors: (1) one systematic scan
is polynomially better than random scan and (2) systematic scans using different
permutations have polynomial differences in mixing times.  We now describe this
model in detail.

\paragraph{Variables}
There are $n$ binary variables $x_1,\ldots,x_n$.
Independently, each variable has a very strong prior of being true.
However, variable $x_i$ is never true unless $x_{i-1}$ is also true. 
The unnormalized probability distribution is the following, where $M$ is 
a very large constant.
\begin{align*}
P(x) \propto
\begin{cases}
0 & \textrm{if $x_i$ is true and $x_{i-1}$ is false for some $i\in\{2,\ldots,n\}$} \\
M^{|x|} & \textrm{otherwise}
\end{cases}
\end{align*}
\paragraph{State Space}
There are $n+1$ states with non-zero probability: $s_0,\ldots,s_n$, where $s_i$ is the state where the first $i$ variables are true and the remaining $n-i$ variables are false.
In the stationary distribution, $s_n$ has almost all of the mass due to the strong priors on the variables, so reaching $s_n$ is essentially equivalent to mixing.
Notice that sampling $x_i$ will almost always move the state from $s_{i-1}$ to $s_i$, very rarely move it from $s_{i}$ to $s_{i-1}$, and can have no other effect.
The worst-case starting state is $s_0$, where the variables must be sampled in the order $x_1,\ldots,x_n$ for this model to mix.

\paragraph{Random Scan}
The number of steps needed to transition from $s_0$ to $s_1$ is distributed as a geometric random variable with mean $n$ (variables are randomly selected, and specifically $x_1$ must be selected).
Similarly, the number of steps needed to transition from $s_{i-1}$ to $s_i$ is distributed as a geometric random variable with mean $n$.
In total, there are $n$ transitions, so $O(n^2)$ steps are needed to mix.


\paragraph{Best Systematic Scan}
The best systematic scan uses the order $x_1,x_2,\ldots,x_n$.
For this scan, one sweep will reach $s_n$ no matter what the starting state is, so the mixing time is $n$.

\paragraph{Worst Systematic Scan}
The worst systematic scan uses the order $x_n,x_{n-1},\ldots,x_1$.
The first sweep only uses $x_1$, the second sweep only uses $x_2$, and in general, any sweep only makes progress using one transition.
Finally, in the $n$-th sweep, $x_n$ is used in the first step.
Thus, this process mixes in $n(n-1) + 1$ steps, which is $O(n^2)$.

\subsection{Two Islands}
\label{ssTwoIslands}

With the sequence of dependencies model, we showed that a single good
permutation can mix much faster than other scan orders.  Next, we describe 
the two islands model (Figure \ref{fig:island}),
which has the reverse behavior: it has bad
permutations that yield much slower mixing times.  The two islands model
achieves this by having two disjoint blocks of variables such that
consecutively sampling two variables from the same block accomplishes very
little.  As a result, a systematic scan that uses a permutation that frequently consecutively samples
from the same block mixes a polynomial factor slower
than both random scan and most other systematic scans.
We now describe this
model in detail.


\paragraph{Variables}

There are $2n$ binary variables grouped into two blocks: $x_1,\ldots,x_{n}$ and $y_1,\ldots,y_n$.
Conditioned on all other variables being false, each variable is equally likely to be true or false.
However, the $x$ variables and the $y$ variables contradict each other.
As a result, if any of the $x$'s are true, then all of the $y$'s must be false,
and if any of the $y$'s are true, then all of the $x$'s must be false.
The unnormalized probability distribution for this model is the following.
\begin{align}
  \label{eqnDistTwoIslands}
  P(x,y) \propto
  \begin{cases}
  0 & \textrm{if $\exists x_i$ true and $\exists y_j$ true} \\
  1 & \textrm{otherwise}
  \end{cases}
\end{align}

This model can be interpreted as a machine learning inference problem in the
following way.
Each variable represents whether the reasoning in some sentence is sound.
The sentences corresponding to $x_1,\ldots,x_n$ and the sentences
corresponding to $y_1,\ldots,y_n$ reach contradicting conclusions.
If any variable is true, its conclusion is correct, so all of 
the sentences that reached the opposite conclusion must be not be sound, and
their corresponding variables must be false.
However, this does not guarantee that all other sentences that reached the
same conclusion have sound reasoning, so it is possible for some variables
in a block to be true while others are false.  Under these assumptions 
alone, the natural way to model this system is with the two islands
distribution in (\ref{eqnDistTwoIslands}).

\paragraph{State Space}

We can think of the states as being divided into three groups: states in
island $x$ (at least one of the $x$ variables are true), states in island
$y$ (at least one of the $y$ variables are true), and a single
\emph{bridge state} $b$ (all variables are false).
The islands are well-connected internally, but it is impossible to directly
move from one island to the other -- the only way to move from one island to
the other is through the bridge.
To simplify the analysis, we assume that the bridge state has very low mass.
This assumption allows us to assume that islands mix rapidly in comparison to
the time required to
move onto the bridge and that chains always move off of the bridge when a
variable is sampled.
The same asymptotic behavior results when the bridge state has the same mass
as the other states.

The bridge is the only way to move from one island to the other, so it acts as
a bottleneck.
As a result, the efficiency of bridge usage is critical to the mixing time.
We will use \textit{bridge efficiency} to refer to the probability that the
chain moves to the other island when it reaches the bridge.
Under the assumption that mixing within the islands is rapid in comparison to
the time needed to move onto the bridge, the mixing time will be inversely
proportional to the bridge efficiency of the chain.

\paragraph{Random Scan}
In random scan, the variable selected after getting on the bridge is
independent of the previous variable.
As a result, with probability $1/2$, the chain will move onto the other
island, and with probability $1/2$, the chain will return to the same
island, so the bridge efficiency is $1/2$.

\paragraph{Best Systematic Scan}
Several different systematic scans achieve the fastest mixing time.
One such scan is $x_1,y_1,x_2,y_2,\ldots,x_n,y_n$.
Since the sampled variables alternate between the blocks,
if the chain moves onto the bridge (necessarily by sampling a
variable from the island it was previously on), it will always proceed to
sample a variable from the other block, which will cause it to move onto the
other island.  Thus, the bridge efficiency is $1$.
More generally, any systematic scan that alternates between sampling from $x$
variables and sampling from $y$ variables will have a bridge efficiency of $1$.

\paragraph{Worst Systematic Scan}
Several different systematic scans achieve the slowest mixing time.
One such scan is $x_1,\ldots,x_n,\, y_1,\ldots,y_n$.
In this case, if the chain moves onto the bridge, it will almost always proceed
to sample a variable from the same block, and return
to the same island.
In fact, the only way for this chain to move across islands is if it moves
from island $x$ to the bridge using transition $x_n$ and then moves to island
$y$ using transition $y_1$, or if it moves from island $y$ to the
bridge using
transition $y_n$ and then moves to island $x$ using transition $x_1$.
Thus, only $2$ of the $2n$ transitions will cross the bridge, and the
bridge efficiency is $1/n$.
More generally, any systematic scan that consecutively samples all $x$
variables and then all $y$ variables
will have a bridge efficiency of $1/n$.

\paragraph{Comparison of Mixing Times}
The mixing times of the chains are inversely proportional to the bridge
efficiency.
As a result, random scan takes twice as long to mix as the best
systematic scan, and mixes $n/2$ times faster than the worst systematic scan.

\subsection{Discrete Pyramid}

In the discrete pyramid model (Figure \ref{fig:pyramid})
introduced by \citet{roberts2015jsp}, there are $n$ binary variables $x_i$, and
the mass is uniformly distributed over all states where at most one
$x_i$ is true. In this model, the mixing time of random scan, $O(n)$, is
asymptotically better than that of systematic scan for any permutation, which all have the same mixing time, $O(n^3)$.

\section{Experiments}
\label{sec:experiments}

In this section, we run several experiments to illustrate the effect of scan
order on mixing times.  First, in Figure \ref{fig:tmix}, we plot the mixing
times of the models from Section \ref{sec:examples} as a function of the number
of variables. These experiments validate our results about the asymptotic
scaling of the mixing time, as well as show that the scan order
can have a significant effect on the mixing time for even small models.
(Due to the exponential state space of the two islands model, we modify it
slightly to make the computation of mixing times feasible: we simplify the
model by only considering the
states that are adjacent to the bridge, and assume that the states on each
individual island mix instantly.)

In the following experiments, we consider a modified version of the 
two islands model, in which the mass of the bridge state is set to 0.1 of the
mass of the other states to allow the
effect of scan order to be clear even for a small number of
variables.
Figure \ref{fig:marginal} illustrates the rate at which different scan orders
explore this modified model.
Due to symmetry, we know that half of the mass should be on each island in
the stationary distribution, so getting half of the mass onto the other island
is necessary for mixing.
This experiment illustrates that random scan and a good systematic scan move to
the other island quickly, while a bad systematic scan requires many more
iterations.

Figure \ref{fig:sorted} illustrates the effect that the permutation chosen for
systematic scan can have on the mixing time.
In this experiment, the mixing time for each permutation was found and plotted in sorted order.
For the sequence of dependencies model, there are a small number of good permutations
which mix very quickly compared to the other permutations and random scan.
However, no permutation is bad compared to random scan.
In the two islands model, as we would expect based on the 
analysis in Section \ref{sec:examples}, there are a small number of bad
permutations which mix very slowly compared to the other permutations and
random scan.
Some permutations are slightly better than random scan, but none of the scan
orders are substantially better.
In addition, the mixing times for systematic scan approximately discretized
due to the fact that mixing time depends so heavily on the bridge efficiency.

\iftoggle{nips}{
  \begin{figure}[tb]
  \begin{center}
  \begin{subfigure}[b]{0.786\textwidth}
  \begin{center}
  \begin{small}
  \begin{tikzpicture}
  \begin{axis}[
    title = Sequence of Dependencies,
    xlabel = $n$,
    ylabel = $t_{\textrm{mix}}$ (thousands),
    xmin=0,
    xmax=50,
    xtick = {0,25,50},
    ymin = -0.1,
    ymax = 1,
    scale = 0.40
  ]
  \addplot[red,mark= ,y filter/.code={\pgfmathparse{#1*0.001}\pgfmathresult}] table [x=n, y=r, col sep=comma] {sequence.csv};
  \addplot[blue, mark= ,y filter/.code={\pgfmathparse{#1*0.001}\pgfmathresult}] table [x=n, y=b, col sep=comma] {sequence.csv};
  \addplot[cyan, mark= ,y filter/.code={\pgfmathparse{#1*0.001}\pgfmathresult}] table [x=n, y=w, col sep=comma] {sequence.csv};
  \end{axis}
  \end{tikzpicture}
  \begin{tikzpicture}
  \begin{axis}[
    title = Two Islands,
    xlabel = $n$,
    xmin=0,
    xmax=50,
    xtick = {0,25,50},
    ymin = -0.1,
    ymax = 1,
    yticklabels={},
    scale = 0.40
  ]
  \addplot[red,y filter/.code={\pgfmathparse{#1*0.001}\pgfmathresult}] table [x=n, y=r, col sep=comma] {islands.csv};
  \addplot[blue ,y filter/.code={\pgfmathparse{#1*0.001}\pgfmathresult}] table [x=n, y=b, col sep=comma] {islands.csv};
  \addplot[cyan ,y filter/.code={\pgfmathparse{#1*0.001}\pgfmathresult}] table [x=n, y=w, col sep=comma] {islands.csv};
  \end{axis}
  \end{tikzpicture}
  \begin{tikzpicture}
  \begin{axis}[
    title = Discrete Pyramid,
    xlabel = $n$,
    xmin=0,
    xmax=50,
    xtick = {0,25,50},
    ymin = -0.1,
    ymax = 1,
    yticklabels={},
    scale = 0.40
  ]
  \addplot[red,mark= ,y filter/.code={\pgfmathparse{#1*0.001}\pgfmathresult}] table [x=n, y=r, col sep=comma] {pyramid.csv};
  \addplot[cyan, mark= ,y filter/.code={\pgfmathparse{#1*0.001}\pgfmathresult}]  table [x=n, y=s, col sep=comma] {pyramid.csv};
  \addplot[blue,dash pattern=on 6pt off 6pt, mark= ,y filter/.code={\pgfmathparse{#1*0.001}\pgfmathresult}]  table [x=n, y=s, col sep=comma] {pyramid.csv};
  \end{axis}
  \end{tikzpicture}
  \end{small}
  \caption{Mixing times for $\epsilon = 1/4$.}
  \label{fig:tmix}
  \end{center}
  \end{subfigure}
  \begin{subfigure}[b]{0.208\textwidth}
  \begin{scriptsize}
  \begin{tikzpicture} 
  \begin{axis}[
    hide axis,
    xmin=10,
    xmax=50,
    ymin=0,
    ymax=0.4,
    scale = 0.40,
    legend style={draw=white!15!black},
    legend cell align=left,
  ]
      \addlegendimage{blue,mark=}
      \addlegendentry{Best Systematic};
      \addlegendimage{cyan,mark=}
      \addlegendentry{Worst Systematic};
      \addlegendimage{brown!75!black}
      \addlegendentry{Other Systematic};
      \addlegendimage{red}
      \addlegendentry{Random};
      \addlegendimage{black,dashed,thick}
      \addlegendentry{True Value};
  \end{axis}
  \end{tikzpicture}
  \end{scriptsize}
  \vspace{1.55cm}
  \end{subfigure}
  \\ \vspace{0.3cm}
  \begin{subfigure}[b]{0.335\textwidth} 
  \begin{center}
  \begin{small}
  \begin{tikzpicture}
  \begin{axis}[
    title = {Two Islands ($n = 10$)},
    xlabel = Iterations (thousands),
    ylabel = Mass on Island,
    xmin=0,
    xmax=100,
    xtick = {0, 25, 50, 75, 100},
    ymin = 0,
    ymax = 0.6,
    ytick = {0, 0.50},
    ytick = {0, 0.25, 0.50},
    scale = 0.40
  ]
  \addplot[red ,x filter/.code={\pgfmathparse{#1*0.001}\pgfmathresult}] table [x=n, y=r1, col sep=comma] {island_marginal_trunc.csv};
  \addplot[blue,x filter/.code={\pgfmathparse{#1*0.001}\pgfmathresult}] table [x=n, y=b1, col sep=comma] {island_marginal_trunc.csv};
  \addplot[cyan,x filter/.code={\pgfmathparse{#1*0.001}\pgfmathresult}] table [x=n, y=w1, col sep=comma] {island_marginal_trunc.csv};
  \addplot[black,thick,dashed] coordinates {(0, 0.5) (100, 0.5)};
  \end{axis}
  \end{tikzpicture}
  \end{small}
  \caption{Marginal island mass over time.}
  \label{fig:marginal}
  \end{center}
  \end{subfigure}
  \begin{subfigure}[b]{0.655\textwidth} 
  \begin{center}
  \begin{small}
  \begin{tikzpicture}
  \begin{axis}[
    title = {Sequence of Dependencies ($n = 10$)},
    xlabel = Percentile,
    ylabel = $t_{\textrm{mix}}$,
    xmin=-5,
    xmax=105,
    xtick = {0, 25, 50, 75, 100},
    ymin = 0,
    ymax = 150,
    scale = 0.40
  ]
  \addplot[red,x filter/.code={\pgfmathparse{#1/36290}\pgfmathresult}] table [x=n, y=r, col sep=comma] {sequence2.csv};
  \addplot[brown!75!black ,x filter/.code={\pgfmathparse{#1/36290}\pgfmathresult}] table [x=n, y=s, col sep=comma] {sequence2.csv};
  \addplot[blue,mark=*,mark size = 0.8] coordinates {(0,10)};
  \addplot[cyan,mark=*,mark size = 0.8] coordinates {(100,91)};
  \end{axis}
  \end{tikzpicture}
  \begin{tikzpicture}
  \begin{axis}[
    title = {Two Islands ($n = 6$)},
    xlabel = Percentile,
    ylabel = $t_{\textrm{mix}}$ (thousands),
    xmin=-5,
    xmax=105,
    xtick = {0, 25, 50, 75, 100},
    ymin = 0,
    ymax = 3,
    scale = 0.40
  ]
  \addplot[red,x filter/.code={\pgfmathparse{#1/4790016}\pgfmathresult},y filter/.code={\pgfmathparse{#1*0.001}\pgfmathresult}] table [x=n, y=r, col sep=comma] {island_seq.csv};
  \addplot[brown!75!black ,x filter/.code={\pgfmathparse{#1/4790016}\pgfmathresult},y filter/.code={\pgfmathparse{#1*0.001}\pgfmathresult}] table [x=n, y=s, col sep=comma] {island_seq.csv};
  \addplot[blue,mark=*,mark size = 0.8] coordinates {(0,0.527)};
  \addplot[cyan,mark=*,mark size = 0.8] coordinates {(100,2.635)};
  \end{axis}
  \end{tikzpicture}
  \end{small}
  \caption{Sorted mixing times of different permutations ($\epsilon = 1/4$).}
  \label{fig:sorted}
  \end{center}
  \end{subfigure}
  \caption{Empirical analysis of the models.}
  \label{fig:experiments}
  \end{center}
  \end{figure}
}{
  \begin{figure}[tb]
  \begin{center}
  \begin{subfigure}[b]{0.72\textwidth} 
  \begin{center}
  \includegraphics{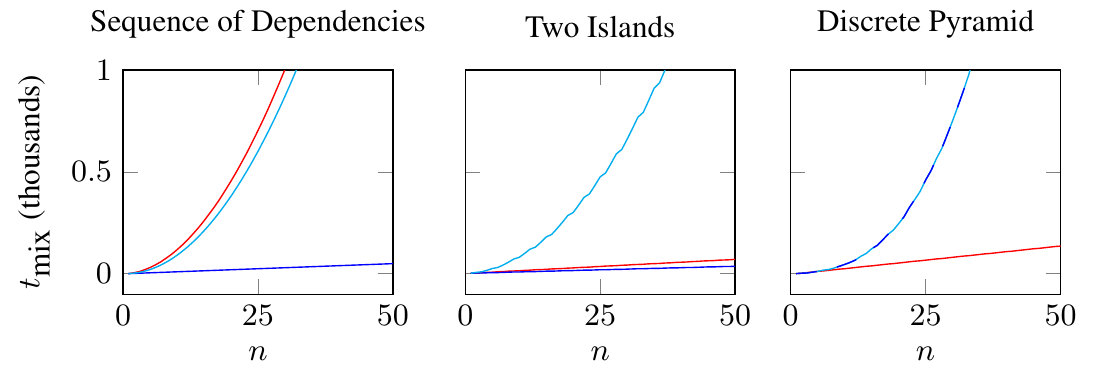}
  \caption{Mixing times for $\epsilon = 1/4$.}
  \label{fig:tmix}
  \end{center}
  \end{subfigure}
  \begin{subfigure}[b]{0.23\textwidth} 
  \includegraphics{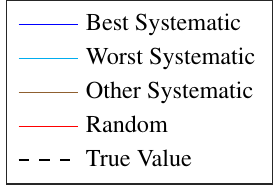}
  \vspace{1.53cm}
  \end{subfigure}
  \\ \vspace{0.3cm}
  \begin{subfigure}[b]{0.33\textwidth} 
  \begin{center}
  \includegraphics{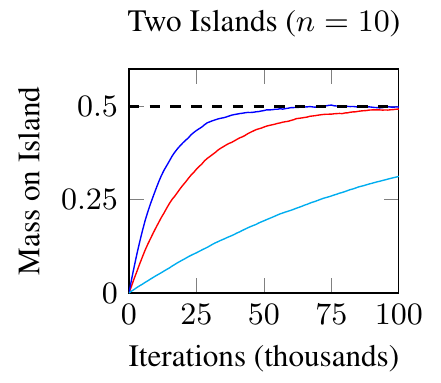}
  \caption{Marginal island mass over time.}
  \label{fig:marginal}
  \end{center}
  \end{subfigure}
  \begin{subfigure}[b]{0.61\textwidth} 
  \begin{center}
  \includegraphics{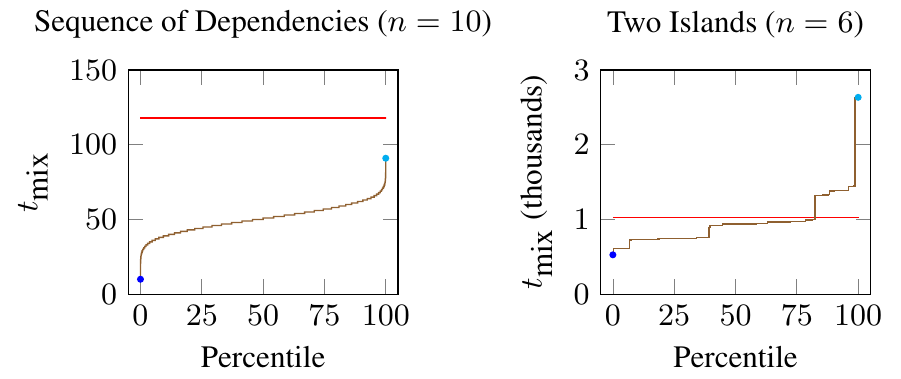}
  \caption{Sorted mixing times of different permutations ($\epsilon = 1/4$).}
  \label{fig:sorted}
  \end{center}
  \end{subfigure}
  \caption{Empirical analysis of the models.}
  \label{fig:experiments}
  \end{center}
  \end{figure}
}

\section{Relative Bounds on Mixing Times via Conductance}
\label{sec:bounds}


In Section \ref{sec:examples}, we described two models for which a systematic scan can mix a polynomial factor faster or slower than random scan, thus invalidating conventional wisdom that the scan order does not have an asymptotically significant effect on mixing times.
This raises a question of how different the mixing times of different scans can be.
In this section, we derive the following weaker -- but correct -- version of
the conjecture stated by \citet{diaconis2013bernoulli} and \citet{levin2009ams}.

One of the obstacles to proving this result is that the systematic scan
chain is not \emph{reversible}.
A standard method of handling non-reversible Markov chains is to study a
lazy version of the Markov chain instead \cite[p.~9]{levin2009ams}.
In the lazy version of a Markov chain, each step has a probability of
$1/2$ of staying at the current state, and acts as a normal step otherwise.
This is equivalent to stopping at a random time that is distributed as
a binomial random variable.
Due to the fact that systematic scan is not reversible, our bounds are on
the lazy systematic scan, rather than the standard systematic scan.

\begin{restatable}{theorem}{thmtmix}
\label{thm:mix}
For any random scan Gibbs sampler $R$ and lazy systematic scan sampler $S$ with
the same stationary distribution $\pi$, their relative mixing times are bounded as follows.
\begin{align*}
\left(1/2-\epsilon\right)^2t_{\textrm{mix}}(R,\epsilon)
&\leq
2t_{\textrm{mix}}^2(S,\epsilon)
\log\left(\frac{1}{\epsilon\pi_{\textrm{min}}}\right)
\\
\left(1/2 - \epsilon\right)^2t_{\textrm{mix}}(S,\epsilon)
&\leq
\frac{8n^2}{\left(\min_{x,i}P_i(x,x)\right)^2}t^2_{\textrm{mix}}(R,\epsilon)
\log\left(\frac{1}{\epsilon\pi_{\textrm{min}}}\right),
\end{align*}
where $P_i$ is the transition matrix corresponding to resampling just
variable $i$,
and $\pi_{\textrm{min}}$ is the probability of the least likely state in
$\pi$.
\end{restatable}

Under mild conditions, namely $\epsilon$ being fixed and 
the quantities $\log(\pi_{\min}^{-1})$ and
$( \min_{x,i} P_i(x,x) )^{-1}$ being at most polynomial in $n$, this theorem
implies that the choice of scan order can only affect the mixing time by
up to polynomial factors in $n$ and $t_{\mathrm{mix}}$.
We now outline the proof of this theorem.



In the two islands models, the mixing time of a scan order was determined by
its ability to move through a single bridge state that restricted flow.
This suggests that a technique with the ability to model the behavior of this
bridge state is needed to bound the relative mixing
times of different scans.
Conductance, also known as the bottleneck ratio, is a topological property of
Markov chains used to bound mixing times by considering the flow of mass
around the model \cite[p.~88]{levin2009ams}.
This ability to model bottlenecks in a Markov chain makes conductance a
natural technique both for studying the two islands model and bounding mixing
times in general.

More formally, consider a Markov chain on state space $\Omega$ with transition matrix $P$ and stationary distribution $\pi$.
The \textit{conductance} of a set $S$ and of the whole chain are respectively defined as
\begin{align*}
\Phi(S) &= \frac{\sum_{x\in S,y \notin S} \pi(x)P(x,y)}{\pi(S)}
&
\Phi_\star &= \min_{S : \pi(S) \leq \frac{1}{2}}\Phi(S).
\end{align*}
Conductance can be directly applied to analyze random scan.
Let $P_i$ be the transition matrix corresponding to sampling variable $i$.
The state space $\Omega$ is used without modification, and the transition matrix is $P = \frac{1}{n}\sum_{i = 1}^{n}P_i$.
The stationary distribution is the expected target distribution $\pi$.

On the other hand, conductance cannot be directly applied to systematic scan.
Systematic scan is not a true Markov chain because it uses a sequence of
transition matrices rather than a single transition matrix.
One standard method of converting systematic scan into a true Markov chain
is to consider each full scan as one step of a Markov chain.
However, this makes it difficult to compare with random scan because it
completely changes which states are connected by single steps of the
transition matrix.
To allow systematic and random scan to be compared more easily, we
introduce an alternative way of converting systematic scan to a true Markov
chain by augmenting the state space.
%
The augmented state space is $\Psi = \Omega \times [n]$, which represents an ordered pair of the normal state and the index of the variable to be sampled.
The transition probability is $P\left((x,i),(y,j)\right) = P_i(x,y)s(i,j)$, where $s(i,j) = \mathbb{I}[i + 1\equiv j\textrm{ (mod $n$)}]$ is an indicator that shows if the correct variable will be sampled next.

Additionally, augmenting the state space for random scan allows easier comparison with systematic scan in some cases.
For augmented random scan, the state space is also
$\Psi = \Omega \times [n]$, the same as for systematic scan.
The transition probability is
$P\left((x,i),(y,j)\right) = \frac{1}{n} P_i(x,y)$, which means
that the next variable to sample is selected uniformly.
The stationary distributions of the augmented random scan and systematic scan
chains are both $\pi\left((x,i)\right) = n^{-1}\pi(x)$.
Because the state space and stationary distribution are the same, augmented
random scan and augmented systematic scan can be compared directly, which
lets us prove the following lemma.

\begin{restatable}{lemma}{thmcond}
\label{thm:condssrs}
For any random scan Gibbs sampler and systematic scan sampler with
the same stationary distribution $\pi$, let $\Phi_{\textrm{RS}}$ denote the
conductance of the random scan process, let $\Phi_{\textrm{RS-A}}$ denote the
conductance of
the augmented random scan process, and let $\Phi_{\textrm{SS-A}}$ denote the
conductance
of the augmented systematic scan process.  Then,
\begin{align*}
\frac{1}{2n}\cdot\min_{x,i}P_i(x,x)\cdot\Phi_{\textrm{RS-A}}
\leq
\Phi_{\textrm{SS-A}}
\leq 
\Phi_{\textrm{RS}}.
\end{align*}
\end{restatable}

In Lemma \ref{thm:condssrs}, the upper bound states that the conductance of
systematic scan is no larger than the conductance of random scan.
We use this in the next section to show that systematic scan cannot mix too
much more quickly than random scan.
To prove this lemma, we show that for any set $S$ under random scan, the
set $\hat S$ containing the corresponding augmented states for systematic scan
will have the same conductance under systematic scan as $S$ had under
random scan.

The lower bound in Lemma \ref{thm:condssrs} states that the conductance of
systematic scan is no smaller than a function of the conductance of
augmented random scan.
This function depends on the number of variables $n$ and
$\min_{x,i} P_i(x,x)$, which is the minimum holding probability of any state.
To prove this lemma, we show that for any set $S$ under augmented systematic
scan, we can bound its conductance under
augmented random scan.

There are well-known bounds on the mixing time of a Markov chain in terms
of its conductance, which we state
in Theorem \ref{thm:bounds} \cite[pp.~89,~235]{levin2009ams}.
\begin{restatable}{theorem}{levinthm}
\label{thm:bounds}
For any lazy or reversible Markov chain,
\begin{align*}
\frac{1/2 - \epsilon}{\Phi_\star}
\leq
t_{\textrm{mix}}(\epsilon)
\leq \frac{2}{\Phi_\star^2}
\log\left(\frac{1}{\epsilon\pi_{\textrm{min}}}\right).
\end{align*}
\end{restatable}

It is straightforward to prove the result of Theorem \ref{thm:mix} by combining these bounds with the conductance bounds from the previous section.

\section{Conclusion}
\label{sec:conclusion}

We studied the effect of scan order on mixing times of Gibbs samplers, and found that for particular models, the scan order can have an asymptotic effect on the mixing times.
These models invalidate conventional wisdom about scan order and show that we cannot freely change scan orders without considering the resulting changes in mixing times.
In addition, we found bounds on the mixing times of different scan orders, which replaces a common conjecture about the mixing times of random scan and systematic scan.

\todo{
Open problems:
\begin{itemize}
  \item What are conditions such that scan order doesn't matter (say, no more than a log factor) -- this is distinct from asking what is sufficient for rapid mixing, since we can just have that all scan orders mix terribly.
  \item Does random reshuffling satisfy the same conjecture?
  \item Are there clean bounds on spectral radii?
  \item Can the marginals have substantially different mixing times?
\end{itemize}
}

\subsection*{Acknowledgments}

The authors acknowledge the support of:
DARPA FA8750-12-2-0335;
NSF IIS-1247701;
NSF CCF-1111943;
DOE 108845;
NSF CCF-1337375;
DARPA FA8750-13-2-0039;
NSF IIS-1353606;
ONR N000141210041 and N000141310129;
NIH U54EB020405;
NSF DGE-114747;
DARPA's SIMPLEX program;
Oracle;
NVIDIA;
Huawei;
SAP Labs;
Sloan Research Fellowship;
Moore Foundation;
American Family Insurance;
Google;
and Toshiba.
The views and conclusions expressed in this material are those of the authors and should not be interpreted as necessarily representing the official policies or endorsements, either expressed or implied, of DARPA, AFRL, NSF, ONR, NIH, or the U.S. Government.

\clearpage
\bibliographystyle{plainnat}
\begin{small}
\bibliography{bibliography}
\end{small}

\iftoggle{appendix}{
\clearpage
\appendix

\section{Additional Models in Which Scan Order Matters}
\label{app:examples}

\renewcommand{\arraystretch}{1.2}
\begin{table}[htb]
\begin{center}
\caption{Models Classified by Relative Mixing Times}
\label{table:examples}
\begin{tabular}{cc|c|c|c|}
\multicolumn{2}{c}{} & \multicolumn{3}{c}{$\min_{\alpha}t_{\textrm{mix}}(S_\alpha)$ versus $t_{\textrm{mix}}(R)$} \\
\multicolumn{2}{c}{} & \multicolumn{1}{c}{$\ll$} & \multicolumn{1}{c}{$\approx$} & \multicolumn{1}{c}{$\gg$} \\
\cline{3-5}
$\max_{\alpha}t_{\textrm{mix}}(S_\alpha)$ & $\ll$ & Memorize and Repeat      & $\times$        & $\times$         \\
\cline{3-5}
versus                 & $\approx$ & Sequence of Dependencies & Known Cases & $\times$         \\
\cline{3-5}
$t_{\textrm{mix}}(R)$                       & $\gg$ & Soft Dependencies        & Two Islands     & Discrete Pyramid \cite{roberts2015jsp} \\
\cline{3-5}
\end{tabular}
\end{center}
\end{table}

Table \ref{table:examples} classifies the models based on the mixing times of the best systematic scan and worst systematic scan relative to random scan.
For models in the first column, the best systematic scan mixes asymptotically faster than random scan.
For models in the second column, they differ only up to logarithmic factors.
For models in the third column, the best systematic scan mixes asymptotically slower than random scan.
The rows are classified in the same way based on the worst systematic scan instead of the best systematic scan.

The models in Section \ref{sec:examples} display possible problems that can occur when changing scan orders.
However, two behaviors were not shown by the models in Section \ref{sec:examples}: a case where random scan is asymptotically worse than all systematic scans, and a case where random scan is asymptotically better than a systematic scan and asymptotically worse than another systematic scan.
This section gives two additional models that exhibit these behaviors, along with a more detailed explanation of the discrete pyramid model.
These two models are more complicated than the models in Section \ref{sec:examples} due to the fact that they require the number of states for each variable to grow with $n$.
Combined with the models in Section \ref{sec:examples}, these models show that all logically consistent asymptotic behaviors (that is, behaviors where the worst systematic scan is no better than the best systematic scan) are possible.

\subsection{Discrete Pyramid}

The discrete pyramid model (Figure \ref{fig:pyramid}) was first described by \citet{roberts2015jsp} and is included for completeness to show that it is possible for random scan to mix asymptotically faster than all systematic scans.


\paragraph{Variables}
There are $n$ binary variables $x_1,\ldots,x_n$.
Conditioned on all other variables being false, each variable is equally likely to be true or false.
However, the variables are all contradictory, and at most one variable can be true.
This model can also be interpreted as an $n$ islands model, where there are $n$ well-connected regions (which consist of a single state) connected by a single bridge state.

\paragraph{State Space}
There are $n+1$ states $s_0,s_1,\ldots,s_n$ with nonzero probability.
$s_0$ is the state where all variables are false, and for all other $i$, $s_i$ is the state where variable $x_i$ is true and all other variables are false.

\paragraph{Random Scan}
The worst-case total variation occurs when the starting state is not $x_0$.
Suppose that the starting state is $s_k$.
$x_k$ will be selected with probability $1/n$, and in this case, the state will change to $x_0$ with probability $1/2$.
Thus, the number of samples needed to leave $x_k$ is distributed as a Geometric random variable with a mean of $2n$.
This means that within $O(n)$ steps, the chain will have left the initial state with high probability.
Once the chain has reached $x_0$, each step has a probability of $1/2$ of leaving $x_0$ and uniformly going to another state.
Thus, $O(1)$ steps are sufficient for the chain to mix after it reaches $x_0$.
In total, the chain mixes in $O(n)$ updates.

\paragraph{All Systematic Scans}
Once again, the worst-case total variation occurs when the starting state is not $x_0$.
Suppose that the starting state is $s_k$.
A systematic scan step will change nothing until $x_k$ is reached.
Then, with probability $1/2$, the state will remain $s_k$, and with probability $1/2$ the state will change to $x_0$.
If it does move to $x_0$, the scan will continue, and move away from $x_0$ with probability $1/2$ at each step. 
Thus, each time the scan reaches the current state, the state will advance $Z$ steps, where $Z$ is a Geometric random variable with rate $1/2$.
This corresponds to a weighted random walk on a circle, which is known to have a mixing time of $O(n^2)$.
\todo{(TODO: need to cite paper on how random walk on circle is $O(n^2)$)}
Notice that each step of the random walk actually requires one full sweep, so $O(n^3)$ steps are needed for any systematic scan to mix. 

\begin{figure}[tb]
\begin{center}
\iftoggle{nips}{
  \begin{tikzpicture}
  \newcommand*{\size}{0.3}
  \newcommand*{\x}{1.2}
  \newcommand*{\y}{1.1}

  \node[circle, minimum size=\size, draw] at (        0,  0) (0)  {};

  \node[circle, minimum size=\size, draw] at ( 1 * \x,  1 * \y) (1) {};
  \node[circle, minimum size=\size      ] at ( 1 * \x,  0 * \y) (i) {$\vdots$};
  \node[circle, minimum size=\size, draw] at ( 1 * \x, -1 * \y) (n) {};

  \node[circle, minimum size=\size, draw] at ( 2 * \x,  1.3 * \y) (12) {};
  \node[circle, minimum size=\size      ] at ( 2 * \x,  1.0 * \y) (1i) {$\vdots$};
  \node[circle, minimum size=\size, draw] at ( 2 * \x,  0.5 * \y) (1n) {};
  \node[circle, minimum size=\size      ] at ( 2 * \x,  0   * \y) (ii) {$\vdots$};
  \node[circle, minimum size=\size, draw] at ( 2 * \x, -0.7 * \y) (n1) {};
  \node[circle, minimum size=\size      ] at ( 2 * \x, -1.0 * \y) (ni) {$\vdots$};
  \node[circle, minimum size=\size, draw] at ( 2 * \x, -1.5 * \y) (nn) {};

  \node[circle, minimum size=\size      ] at ( 3 * \x,  1.3 * \y) (12i) {$\cdots$};
  \node[circle, minimum size=\size      ] at ( 3 * \x,  0.5 * \y) (1ni) {$\cdots$};
  \node[circle, minimum size=\size      ] at ( 3 * \x, -0.7 * \y) (n1i) {$\cdots$};
  \node[circle, minimum size=\size      ] at ( 3 * \x, -1.5 * \y) (nni) {$\cdots$};

  \node[circle, minimum size=\size, draw] at ( 4 * \x,  1.3 * \y) (12ii) {};
  \node[circle, minimum size=\size, draw] at ( 4 * \x,  0.5 * \y) (1nii) {};
  \node[circle, minimum size=\size, draw] at ( 4 * \x, -0.7 * \y) (n1ii) {};
  \node[circle, minimum size=\size, draw] at ( 4 * \x, -1.5 * \y) (nnii) {};

  \node[circle, minimum size=\size      ] at ( 5 * \x,  1.3 * \y) (12iii) {$\cdots$};
  \node[circle, minimum size=\size      ] at ( 5 * \x,  0.5 * \y) (1niii) {$\cdots$};
  \node[circle, minimum size=\size      ] at ( 5 * \x, -0.7 * \y) (n1iii) {$\cdots$};
  \node[circle, minimum size=\size      ] at ( 5 * \x, -1.5 * \y) (nniii) {$\cdots$};

  \node[circle, minimum size=\size, draw] at ( 6 * \x,  0 * \y) (x) {};

  \draw[black, -{Stealth[scale=1.2]}] (0) -- (1) node[midway,above]{$T_1$};
  \draw[black, -{Stealth[scale=1.2]}] (0) -- (n) node[midway,below]{$T_n$};

  \draw[black, -{Stealth[scale=1.2]}] (1) -- (12) node[midway,above]{$T_2$};
  \draw[black, -{Stealth[scale=1.2]}] (1) -- (1n) node[midway,below]{$T_n$};
  \draw[black, -{Stealth[scale=1.2]}] (n) -- (n1) node[midway,above]{$T_1$};
  \draw[black, -{Stealth[scale=1.2]}] (n) -- (nn) node[midway,below]{$T_{n-1}$};

  \draw[black, -{Stealth[scale=1.2]}] (12) -- (2.7 * \x,  1.1 * \y);
  \draw[black, -{Stealth[scale=1.2]}] (12) -- (2.7 * \x,  1.5 * \y);
  \draw[black, -{Stealth[scale=1.2]}] (1n) -- (2.7 * \x,  0.3 * \y);
  \draw[black, -{Stealth[scale=1.2]}] (1n) -- (2.7 * \x,  0.7 * \y);
  \draw[black, -{Stealth[scale=1.2]}] (n1) -- (2.7 * \x, -0.9 * \y);
  \draw[black, -{Stealth[scale=1.2]}] (n1) -- (2.7 * \x, -0.5 * \y);
  \draw[black, -{Stealth[scale=1.2]}] (nn) -- (2.7 * \x, -1.7 * \y);
  \draw[black, -{Stealth[scale=1.2]}] (nn) -- (2.7 * \x, -1.3 * \y);

  \draw[black, -{Stealth[scale=1.2]}] (12i) -- (12ii);
  \draw[black, -{Stealth[scale=1.2]}] (1ni) -- (1nii);
  \draw[black, -{Stealth[scale=1.2]}] (n1i) -- (n1ii);
  \draw[black, -{Stealth[scale=1.2]}] (nni) -- (nnii);

  \draw[black, -{Stealth[scale=1.2]}] (12ii) -- (12iii);
  \draw[black, -{Stealth[scale=1.2]}] (1nii) -- (1niii);
  \draw[black, -{Stealth[scale=1.2]}] (n1ii) -- (n1iii);
  \draw[black, -{Stealth[scale=1.2]}] (nnii) -- (nniii);

  \draw[black, -{Stealth[scale=1.2]}] (12iii) -- (x);
  \draw[black, -{Stealth[scale=1.2]}] (1niii) -- (x);
  \draw[black, -{Stealth[scale=1.2]}] (n1iii) -- (x);
  \draw[black, -{Stealth[scale=1.2]}] (nniii) -- (x);

\draw [decorate,decoration={brace,amplitude=10pt}] (4 * \x,-1.8 * \y) -- (     0,-1.8 * \y) node[black,midway,yshift=-0.55cm] {Memorize};
\draw [decorate,decoration={brace,amplitude=10pt}] (6 * \x,-1.8 * \y) -- (4 * \x,-1.8 * \y) node[black,midway,yshift=-0.55cm] {Repeat};





\end{tikzpicture}
}{
  \includegraphics{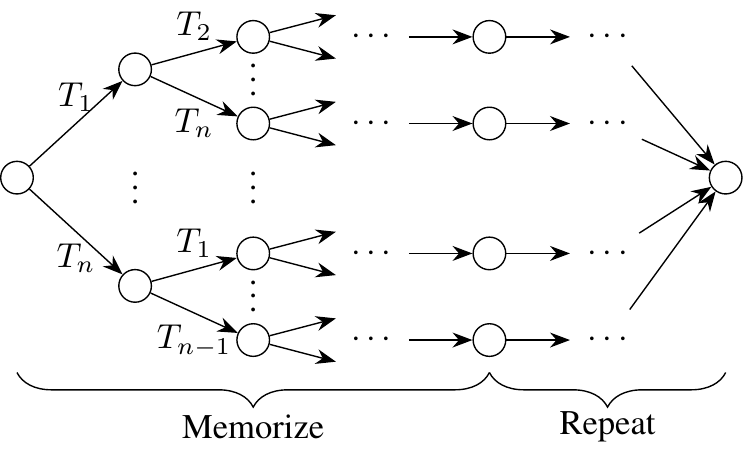}
}
\caption{Memorize and Repeat Model}
\label{fig:memorize}
\end{center}
\end{figure}

\subsection{Memorize and Repeat}
We introduce the memorize and repeat model to show that it is possible for all systematic scans to be asymptotically better than random scan.

\paragraph{Description}
In this example (Figure \ref{fig:memorize}), there are two types of states: states memorizing a permutation, and states requiring the memorized permutation to be repeated.
In the stationary distribution, almost all of the mass is on the state where a permutation has been memorized and repeated, and the probability of the states increase exponentially from left to right.
There are $n$ variables $x_1,\ldots,x_n$.
In the memorize phase, sampling a variable will allow the state to change as long as the variable has not been sampled before.
In the repeat phase, a sampling a variable will only allow the state to change if it is the next variable in the memorized permutation.
The repeat phase is similar to the sequence of dependencies example, with the variables rearranged to match the memorized permutation.

In the worst-case starting distribution, the wrong permutation can already be memorized, which results in no asymptotic gap between mixing times.
To allow the difference in mixing time to appear even in the worst-case analysis, we repeat the memorize and repeat process $n$ times.

\paragraph{Random Scan}
Random scan encounters the coupon collector problem to get through the memorize phase, so $O(n\log n)$ steps are needed.
Afterwards, the repeat phase is equivalent to a sequence of dependencies, which requires $O(n^2)$ steps.
Thus, to get through one memorize and repeat chain, $O(n^2)$ steps are needed.
This mechanism is repeated $n$ times, so $O(n^3)$ steps are needed.

\paragraph{All Systematic Scans}
If a systematic scan starts at the beginning of a memorize sequence, $2n$ steps are sufficient to move through the memorize and repeat chain ($n$ steps to memorize, $n$ steps to repeat).
However, when analyzing the worst-case total variation, we must consider the case where the wrong permutation has already been memorized, which requires $O(n^2)$ steps to move through.
However, this can only happen once, so in the worst case $O(n^2)$ steps are needed to move through the one incorrect permutation, and $O(n)$ steps are needed for the remaining $n-1$ memorize and repeat chains.
Thus, $O(n^2)$ steps are needed in total.

\paragraph{State Space}

To simplify the explanation of the state space and formulation as a Gibbs sampler, this explanation is for only one memorize and repeat cycle.
In the memorize phase, the states are labeled by the part of the permutation that has already been memorized.
In the repeat phase, the states are labeled by the remainder of the permutation that still has to be repeated.

\paragraph{Formulation as a Gibbs Sampler}
There are $n$ integer-valued variables $x_1,\ldots,x_n$ with values ranging from $0$ to $n+1$.
In the initial state that does not have anything memorized, all variables have a value of 0.
In the memorize phase, any variable that has not yet been memorized has a value of 0, and any variable that has been memorized stores its index within the permutation. 
This means that for states in the memorize phase with $i$ variables memorized, exactly one variable will have a value of $j$, for each integer $j$ from 1 to $i$, and the remaining $n-i$ variables that have not been memorized have a value of $0$.
As more variables are memorized, the states become exponentially more likely.
Next, in the repeat phase, when a variable is used, its value changes to $n+1$.
This means that for states in the repeat phase with $i$ variables repeated, the $i$ variables that have been repeated have a value of $n+1$, and the remaining $n-i$ variables that still have to be repeated have a unique integer from $i+1$ to $n$.
As more variables are repeated, the states become exponentially more likely.
The log probability distribution is the following, where $M$ is a very large constant, $Z$ is the normalizing constant, and $\sigma(a_1,\ldots,a_n)$ denotes the set of all permutations of $(a_1,\cdots,a_n)$.
\begin{align*}
\log_M P(x)
&= -\log_M Z + 
\begin{cases}
\max(x_1,\ldots,x_n) & \textrm{if $(x_1,\ldots,x_n)\in\sigma(0,\ldots,0,\ 1,\ldots, i)$} \\
n + \sum_{i=1}^{n}\mathbb{I}[x_i = n + 1] & \textrm{if $(x_1,\ldots,x_n)\in\sigma(i+1,\ldots,n,\ n+1,\ldots, n+1)$} \\
-\infty & \textrm{otherwise}
\end{cases} \\
&= -\log_M Z + 
\begin{cases}
\left(\textrm{Number of Memorized Variables}\right) & \textrm{if valid memorize state} \\
n + \left(\textrm{Number of Repeated Variables}\right) & \textrm{if valid repeat state} \\
-\infty & \textrm{otherwise}
\end{cases}
\end{align*}

First, for states in the memorize phase, if a variable that has not yet been memorized is sampled, then its value will almost always change to $i+1$, where $i$ is the number of variables already memorized, due to the large value of $M$.
However, sampling a variable that has already been memorized will not change its value because all other values will either result in an invalid state (and have a probability of 0) or be significantly less likely due to the large value of $M$.
Next, for states in the repeat phase, if the correct variable is sampled, then its value will almost always change to $n+1$.
However, sampling the wrong variable will not change its value because all other values result in an invalid state or decrease the probability of the state.

\subsection{Soft Dependencies}

In the soft dependencies model, some systematic scans mix asymptotically faster than random scan, and some systematic scans mix asymptotically slower.

\paragraph{Description}
This example resembles the memorize portion of the previous example.
However, no repeat phase is included, and only some permutations are accepted.
In particular, a permutation is accepted only if each element of the permutation is followed by an element that is in the next $\sqrt{n}$ unused variables that come after it (mod $n$).

More formally, consider a permutation $(a_1,\ldots,a_n)$ of $(1,\ldots,n)$.
The permutation is accepted only if the following holds for all $i\in\{1,\ldots,n-1\}$.
\begin{align*}
\sqrt{n}
&\geq
\left\vert\left\{\ j \in \{1,\ldots,n\} \mid j > i+1\textrm{ and } a_j - a_i\textrm{ (mod $n$)} < a_{i+1} - a_i\textrm{ (mod $n$)}\ \right\}\right\vert
\end{align*}
In this condition, the requirement that $j > i + 1$ means that $a_j$ is an unused element, and the requirement that $a_j - a_i\textrm{ (mod $n$)} < a_{i+1} - a_i\textrm{ (mod $n$)}$ means that starting from $a_i$, $a_j$ is reached before $a_{i+1}$.
These two requirements imply that $a_{i+1}$ is within the next $\sqrt{n}$ unused variables following $a_i$.
This condition is equivalent to the following.
\begin{align*}
\sqrt{n}
&\geq
\begin{cases}
\sum_{j=i+2}^{n}\mathbb{I}[a_i < a_j < a_{i+1}] & \textrm{if $a_i < a_{i+1}$} \\
\sum_{j=i+2}^{n}\mathbb{I}[a_i < a_j \textrm{ or } a_j < a_{i+1}] & \textrm{if $a_i > a_{i+1}$} \\
\end{cases}
\end{align*}

\todo{In addition, all variables have to be sampled once afterwards. (this is an artifact of the Gibbs sampler -- should I actually include this?). This doesn't affect the $O(\cdot)$ of the mixing times at all -- it just adds some lower order terms.}

\paragraph{Random Scan}
First, consider when there are at least $\sqrt{n}$ remaining transitions.
The first $n - \sqrt{n} = O(n)$ transitions will have this property.
During this period, random scan advances to the next state with probability $n / \sqrt{n} = \sqrt{n}$.
Thus, $O(n\sqrt{n})$ steps are needed to get until there are fewer than $\sqrt{n}$ remaining transitions.
Once there are fewer than $\sqrt{n}$ remaining transitions, random scan needs $O(n)$ steps to make each transition, so $O(n\sqrt{n})$ more steps are needed.
In total, $O(n\sqrt{n})$ steps are needed.

\paragraph{Best Systematic Scan}
The best systematic scan uses the order $x_1,\ldots,x_n$.
This uses $n$ steps to give a valid permutation, and thus mixes in $n$ steps.

\paragraph{Worst Systematic Scan}
The worst systematic scan uses the order $x_n,x_{n-1},\ldots,x_1$.
After a transition is taken, nearly a full scan is needed until the next valid transition will be reached, so $O(n^2)$ steps are needed.

\paragraph{State Space}

The state space has a similar form to the memorize portion of the previous example.
The states are labeled by the portion of the permutation that has already been given, but only prefixes to valid permutations are accepted.

\paragraph{Formulation as a Gibbs Sampler}
There are $n$ integer-valued variables $x_1,\ldots,x_n$ with values ranging from $0$ to $n$.
In the initial state that does not have anything memorized, all variables have a value of 0.
Then, as the permutation is memorized, the variables that have not been used yet have a value of 0, and any variable that has been used stores its index within the permutation.

In addition, to merge the permutations together into one state, all variables need to be sampled one more time (note that this does not require the variables to be sampled in a particular order).
This extra sampling process only adds lower order terms and constant factors to the mixing time.
Once each variable is sampled again, its value is changed to $n$.

The log probability distribution is the following, where $M$ is a very large constant, $Z$ is the normalizing constant, and $\sigma(a_1,\ldots,a_n)$ denotes the set of all permutations of $(a_1,\cdots,a_n)$.
\begin{align*}
\log_M P(x)
&= -\log_M Z + 
\begin{cases}
\max(x_1,\ldots,x_n) & \textrm{if $(x_1,\ldots,x_n)\in\sigma(0,\ldots,0,\ 1,\ldots, i)$ and is valid} \\
n - 1 + \sum_{i=1}^{n}\mathbb{I}[x_i = n] & \textrm{if at least two of $(x_1,\ldots,x_n)$ have a value of $n$} \\
-\infty & \textrm{otherwise}
\end{cases} \\
&= -\log_M Z + 
\begin{cases}
\left(\textrm{Number of Memorized Variables}\right) & \textrm{if valid memorize state} \\
n - 1 + \left(\textrm{Number of Merged Variables}\right) & \textrm{if valid merge state} \\
-\infty & \textrm{otherwise}
\end{cases}
\end{align*}

First, in the memorize phase, sampling a valid unused variable will almost always change its value to $i+1$, where $i$ is the number of variables already memorized, due to the large value of $M$.
However, sampling a variable that has already been memorized or sampling a variable that least to an invalid permutation will not change its value.
Then, when the last variable in the permutation is sampled, its value changes to $n$.
Afterwards, sampling any variable will also change its value to $n$.

\section{Priors in the Sequence of Dependencies}
\label{app:priors}

To simplify the analysis of the sequence of dependencies, we assumed that the priors on the variables were strong enough such sampling the correct variable essentially always caused the state to advance.
We now analyze how large $M$ (the strength of the priors) needs to be in order to allow the best systematic scan to mix in $O(n)$ time.
In this section, we show that if $M = \Omega(n)$, then the best systematic scan mixes in $O(n)$ time, and if $M = o(n)$, then the best systematic scan cannot mix $O(n^2)$ time.

First, suppose that $M = \Omega(n)$.
This means that $M \geq cn$ for some $c > 0$ and sufficiently large $n$.
The probability of transitioning from $s_{i-1}$ to $s_i$ when variable $x_i$ is sampled is
\begin{align*}
\frac{M}{1 + M} \geq \frac{cn}{1 + cn} = 1 - \frac{1}{1 + cn} > 1 - \frac{1}{cn}.
\end{align*}
The probability of transitioning from $s_0$ to $s_n$ after sampling the sequence $x_1,\ldots,x_n$ is at least
$(1 - 1/cn)^n$, which limits to $e^{-1/c}$ for large $n$.
In other words, if $M = \Omega(n)$, then a single sweep of the best systematic scan will reach $s_n$ with a probability that does not approach 0.
Thus, a constant number of sweeps is sufficient to reach $s_n$ with high probability, which is equivalent to mixing.

On the other hand, suppose that $M = o(n)$.
Then, for any $c>0$, $M < cn$ for sufficiently large $n$.
This means that as $n$ increases, the probability that a single sweep of the best systematic scan will reach $s_n$ will go to 0, so no constant number of sweeps will be sufficient to mix.

\section{Bridge Efficiency with Normal Mass on Bridge}
\label{app:bridge}

In the analysis of the two islands model, we assumed that the bridge has negligible mass.
We now analyze the mixing times without the assumption that the bridge has small mass.
The same asymptotic behavior still results.

Even when the bridge has the same mass as the other states, it still acts as a bottleneck to the model.
For sufficiently large $n$, the islands will mix rapidly in comparison reaching the bridge, so the mixing time is still inversely proportional to the bridge efficiency.
However, when the bridge has the same mass as the other states, sampling a variable while on the bridge will only have a $1/2$ chance of moving off of the bridge.

\paragraph{Random Scan}
In random scan, the variables that are sampled are completely independent, so the bridge efficiency is still $1/2$.

\paragraph{Best Systematic Scan}
Consider the scan $x_1,y_1,x_2,y_2,\ldots$.
Suppose the sampling $x_1$ changes the state to the bridge state -- a similar analysis will apply for any other variable.
The next variable is $y_1$, which will change the state to the other island with probability $1/2$.
Afterwards, $x_2$ is sampled, which will change the state to the same island with probability $1/4$.
Then, $y_2$ is sampled, which will change the state to the other island with probability $1/8$.
Thus, the probability of moving to the other island is $1/2 + 1/8 + 1/32 + \ldots = 2/3$.

\paragraph{Worst Systematic Scan}
Consider the scan $x_1,\ldots,x_n,\ y_1,\ldots,y_n$.
Suppose that sampling $x_n$ changes the state to the bridge state.
For large $n$, the probability of leaving the bridge state onto island $y$ approaches 1.
Next, suppose that sampling $x_{n-1}$ changes the state to the bridge state.
There is a probability of $1 / 2$ of moving back to island $x$ via $x_n$, but the chain will move onto island $y$ otherwise.
In general, moving onto the bridge via $x_{i}$ or $y_{i}$ will result in moving to the island with probability $2^{-n + i}$.
The average probability of moving onto the other island is then $2/n$.

\section{Proofs for Section \ref{sec:bounds}}
\label{app:proofs}

In this section, we prove our relative bounds on mixing times (Theorem \ref{thm:mix}), along with related claims and lemmas.
\begin{claim}
The stationary distribution of augmented random and systematic scan is
\begin{align*}
\pi\left((x,i)\right) = \frac{1}{n}\pi(x)
\end{align*}
\end{claim}
\begin{proof}
We prove this claim by showing that applying the transition matrix for augmented random scan or augmented systematic scan does not change this distribution.

For augmented random scan,
\begin{align*}
\sum_{x,i}\pi((x,i))P((x,i),(y,j))
&= \sum_{x,i}\frac{\pi(x)}{n}\cdot P((x,i),(y,j)) 
 = \frac{1}{n}\sum_{x,i}\pi(x)\cdot\frac{1}{n}P_i(x,y) \\
&= \frac{1}{n}\sum_{i=1}^{n}\left[\sum_{x\in\Omega}\left(\pi(x)P_i(x,y)\right)\cdot\frac{1}{n}\right] \\
&= \frac{1}{n}\sum_{i=1}^{n}\left[\pi(y)\cdot\frac{1}{n}\right]
 = \frac{1}{n}\pi(y) \\
&= \pi((y,j))
\end{align*}

For augmented systematic scan,
\begin{align*}
\sum_{x,i}\pi((x,i))P((x,i),(y,j))
&= \sum_{x,i}\frac{\pi(x)}{n}\cdot P((x,i),(y,j)) 
 = \frac{1}{n}\sum_{x,i}\pi(x)\cdot P_i(x,y)s(i,j) \\
&= \frac{1}{n}\sum_{i=1}^{n}\left[\sum_{x\in\Omega}\left(\pi(x)P_i(x,y)\right)s(i,j)\right] \\
&= \frac{1}{n}\sum_{i=1}^{n}\left[\pi(y)s(i,j)\right]
 = \frac{1}{n}\pi(y) \\
&= \pi((y,j))
\end{align*}
\end{proof}



\thmcond*
\begin{proof}
\item
\paragraph{Upper Bound:}
The conductance of the whole chain is the smallest conductance of any set with mass no larger than $\frac{1}{2}$.
Then, to prove that this inequality holds, we will show that for any set $S\in\Omega$ with mass no larger than $\frac{1}{2}$, there exists a set $T\in\Psi$ with mass no larger than $\frac{1}{2}$ such that the conductance of $S$ under random scan is the same as the conductance of $T$ under augmented systematic scan.

From the standpoint of random scan, consider a set $S\in\Omega$ with mass no larger than $\frac{1}{2}$. 
The conductance is
\begin{align*}
\Phi_{\textrm{RS}}(S)
&= \frac{\sum_{x\in S} \sum_{y\in S^c} \pi(x)P(x,y)}{\pi(S)} \\
&= \frac{\frac{1}{n}\sum_{i=1}^{n}\sum_{x\in S} \sum_{y\in S^c} \pi(x)P_i(x,y)}{\pi(S)}
\end{align*}

Then, for augmented systematic scan, consider the set $T = \{\, (x,i) : x\in S, i\in[n]\, \}$.
First, notice that $\pi(T) = \pi(S) \leq \frac{1}{2}$.
The conductance is
\begin{align*}
\Phi_{\textrm{SS-A}}(T)
&= \frac{\sum_{(x,i)\in T} \sum_{(y,j)\in T^c} \pi\left((x,i)\right)P\left((x,i),(y,j)\right)}{\pi(T)} \\
&= \frac{\sum_{i=1}^{n}\sum_{j=1}^{n}\sum_{x\in S} \sum_{y\in S^c} \pi\left((x,i)\right)P\left((x,i),(y,j)\right)}{\pi(S)} \\
&= \frac{1}{n}\frac{\sum_{i=1}^{n}\sum_{j=1}^{n}\sum_{x\in S} \sum_{y\in S^c} \pi(x)P_i(x,y)s(i,j)}{\pi(S)} \\
&= \frac{1}{n}\frac{\sum_{i=1}^{n}\sum_{x\in S} \sum_{y\in S^c} \pi(x)P_i(x,y)}{\pi(S)} \\
&= \Phi_{\textrm{RS}}(S)
\end{align*}

This implies that for any $S\in\Omega$ with $\pi(S) \leq \frac{1}{2}$, there exists $T\in\Psi$ with $\pi(T) \leq \frac{1}{2}$ such that $\Phi_{\textrm{RS}}(S) = \Phi_{\textrm{SS}}(T)$.
Therefore, $\Phi_{\textrm{SS-A}} \leq \Phi_{\textrm{RS}}$.

\paragraph{Lower Bound:}

In this proof, we will work with the \textit{flow} between two sets
\begin{align*}
Q(A,B) &= \sum_{x\in A, y\in B}\pi(x)P(x,y).
\end{align*}
Notice that the conductance of a set can then be defined as
\begin{align*}
\Phi(S) = \frac{Q(S,S^c)}{\pi(S)}.
\end{align*}

To prove that this inequality holds for the whole chain, we will show that the same inequality holds for any set $S\in\Psi$.
Consider some arbitrary state $x\in \Omega$.
Flow can leave from the corresponding augmented states in two ways: flowing from some $(x,i)\in S$ to $(x,j)\in S^c$ or flowing from $(x,i)\in S$ to $(y,j)\in S^c$, where $y\neq x$ ($x$ and $y$ differ in only variable $i$).
Let $S_x = \{\, (x,i) \in S\, \}$, and let $S^c_x = \{\ (x,i) \in S^c\ \}$.
These two components can be written as $Q(S_x,S^c_x)$ and $\sum_{y\neq x}Q(S_x,S^c_y)$.

Now, we will find upper bounds for the random scan flows and lower bounds for the systematic scan flows.
In the following statements, it is implicit that $y\neq x$, and $\gamma = \min_{x,i}P_i(x,x)$ will denote the minimum holding probability.

First, we bound the amount of flow from $(x,i)\in S$ to $(x,j)\in S^c$.
For augmented random scan, the following upper bound holds.
\begin{align*}
Q_{\textrm{RS}}(S_x,S^c_x)
&= \sum_{(x,i)\in S_x,(x,j)\in S^c_x} \pi((x,i))P((x,i),(x,j)) \\
&= \sum_{(x,i)\in S_x,(x,j)\in S^c_x} \frac{1}{n}\pi(x)\cdot \frac{1}{n}P_i(x,x) \\
&\leq \sum_{(x,i)\in S_x,(x,j)\in S^c_x} \frac{1}{n}\pi(x)\cdot \frac{1}{n} \\
&= \frac{|S_x|}{n}\pi(x)\frac{n - |S_x|}{n} \\
&\leq
\begin{cases}
\frac{1}{4}\pi(x) & \textrm{if $|S_x|\neq 0,n$} \\
0                 & \textrm{if $|S_x|= 0,n$}
\end{cases}
\end{align*}
For augmented systematic scan, the following lower bound holds.
\begin{align*}
Q_{\textrm{SS}}(S_x,S^c_x)
&= \sum_{(x,i)\in S_x,(x,j)\in S^c_x} \pi((x,i))P((x,i),(x,j)) \\
&= \sum_{(x,i)\in S_x,(x,j)\in S^c_x} \frac{1}{n}\pi(x)\cdot P_i(x,x)s(i,j) \\
&\geq \sum_{(x,i)\in S_x,(x,j)\in S^c_x} \frac{1}{n}\pi(x) \cdot \gamma\cdot s(i,j) \\
&\geq
\begin{cases}
\frac{1}{n}\pi(x)\gamma & \textrm{if $|S_x|\neq 0,n$} \\
0                       & \textrm{if $|S_x|= 0,n$}
\end{cases}
\end{align*}

Similarly,
\begin{align*}
Q_{\textrm{SS}}(S_y,S^c_y) &\geq
\begin{cases}
\frac{1}{n}\pi(y)\gamma & \textrm{if $|S_y|\neq 0,n$} \\
0                       & \textrm{if $|S_y|= 0,n$}
\end{cases}
\end{align*}

Now, we bound the amount of flow from $x$ to $y$ for $y\neq x$.
Note that $P_i(x,y) = 0$ for all $i$ if $x$ and $y$ differ in more than one variable.
As a result, we will assume that $x$ and $y$ differ in only variable $i$ for the next two bounds.
For augmented random scan, the following upper bound holds. 
\begin{align*}
Q_{\textrm{RS}}(S_x,S^c_y)
&=
\begin{cases}
\frac{1}{n}\pi(x)P_i(x,y)\frac{n-|S_y|}{n} & \textrm{if $(x,i)\in S$} \\
0 & \textrm{if $(x,i)\not\in S$}
\end{cases} \\
&\leq
\begin{cases}
\frac{1}{n}\pi(x)P_i(x,y) & \textrm{if $(x,i)\in S$ and $|S_y^c|\neq 0$} \\
0 & \textrm{if $(x,i)\not\in S$ or $|S_y^c|=0$}
\end{cases}
\end{align*}

In the derivation of the next bound, note that if $|S^c_y| = n$, then we are guaranteed that $(y,i+1 (\textrm{mod }n))\in S^c_y$.
For augmented systematic scan, the following lower bound holds.
\begin{align*}
Q_{\textrm{SS}}(S_x,S^c_y)
&=
\begin{cases}
\frac{1}{n}\pi(x)P_i(x,y) & \textrm{if $(x,i)\in S$ and $(y,i+1)\in S^c$} \\
0 & \textrm{otherwise}
\end{cases} \\
&\geq
\begin{cases}
\frac{1}{n}\pi(x)P_i(x,y) & \textrm{if $(x,i)\in S$ and $|S^c_y| = n$} \\
0 & \textrm{otherwise}
\end{cases} \\
\end{align*}

Now, we can derive several inequalities between the augmented random scan flow and the augmented systematic scan flow as direct consequences of the bounds we just found.
First, we bound the relative flows from $S_x$ to $S_x^c$.
\begin{align*}
Q_{\textrm{SS}}(S_x,S^c_x)
&\geq \frac{4\gamma}{n}Q_{\textrm{RS}}(S_x,S^c_x)
 \geq \frac{\gamma}{n}Q_{\textrm{RS}}(S_x,S^c_x)
\end{align*}

Next, we bound the relative flows from $x$ to $y$, where $x$ and $y$ differ in exactly variable $i$.
\begin{align*}
&\phantom{{}={}}
Q_{\textrm{SS}}(S_x,S^c_y) + \frac{1}{n}P_i(y,x)Q_{\textrm{SS}}(S_y,S_y^c) \\
&\geq
\begin{cases}
\frac{1}{n}\pi(x)P_i(x,y) & \textrm{if $(x,i)\in S$ and $|S^c_y| = n$} \\
\frac{1}{n}P_i(y,x)\cdot\frac{1}{n}\pi(y)\gamma & \textrm{if $(x,i)\in S$ and $|S^c_y| \neq 0,n$} \\
0 & \textrm{otherwise}
\end{cases} \\
&=
\begin{cases}
\frac{1}{n}\pi(x)P_i(x,y) & \textrm{if $(x,i)\in S$ and $|S^c_y| = n$} \\
\frac{1}{n^2}\pi(x)P_i(x,y)\gamma & \textrm{if $(x,i)\in S$ and $|S^c_y| \neq 0,n$} \\
0 & \textrm{otherwise}
\end{cases} \\
&\geq
\begin{cases}
\frac{1}{n^2}\pi(x)P_i(x,y)\gamma & \textrm{if $(x,i)\in S$ and $|S_y^c| \neq 0$} \\
0 & \textrm{otherwise}
\end{cases} \\
&=
\frac{\gamma}{n}
\begin{cases}
\frac{1}{n}\pi(x)P_i(x,y) & \textrm{if $(x,i)\in S$ and $|S_y^c| \neq 0$} \\
0 & \textrm{otherwise}
\end{cases} \\
&\geq \frac{\gamma}{n}\Phi_{\textrm{RS}}(S_x,S^c_y)
\end{align*}

Finally, we bound the relative flows.
\begin{align*}
Q_{\textrm{SS}}(S,S^c)
&= \sum_{x\in\Omega}\sum_{y\in\Omega} Q_{\textrm{SS}}(S_x,S^c_y)
 = \sum_{x\in\Omega}\left(Q_{\textrm{SS}}(S_x,S^c_x) + \sum_{y\neq x} Q_{\textrm{SS}}(S_x,S^c_y)\right) \\
&\geq \frac{1}{2}\sum_{x\in\Omega}\left(2Q_{\textrm{SS}}(S_x,S^c_x) + \sum_{y\neq x} Q_{\textrm{SS}}(S_x,S^c_y)\right) \\
& \frac{1}{2}\sum_{x\in\Omega}\left(Q_{\textrm{SS}}(S_x,S^c_x) + \sum_{y\in\Omega} \sum_{i=1}^{n}\frac{1}{n}P_i(y,x)Q_{\textrm{SS}}(S_y,S^c_y) + \sum_{y\neq x} Q_{\textrm{SS}}(S_x,S^c_y)\right) \\
&\geq \frac{1}{2}\sum_{x\in\Omega}\left(Q_{\textrm{SS}}(S_x,S^c_x) + \sum_{y\neq x} \left(Q_{\textrm{SS}}(S_x,S^c_y) + \sum_{i=1}^{n}\frac{1}{n}P_i(y,x)Q_{\textrm{SS}}(S_y,S^c_y)\right)\right) \\
&\geq \frac{1}{2}\sum_{x\in\Omega}\left(\frac{\gamma}{n}Q_{\textrm{RS}}(S_x,S^c_x) + \frac{\gamma}{n}\sum_{y\neq x} Q_{\textrm{RS}}(S_x,S^c_y)\right) \\
&\geq \frac{\gamma}{2n}\sum_{x\in\Omega}\left(Q_{\textrm{RS}}(S_x,S^c_x) + \sum_{y\neq x} Q_{\textrm{RS}}(S_x,S^c_y)\right) \\
&= \frac{\gamma}{2n}Q_{\textrm{SS}}(S,S^c)
\end{align*}

The mass of $S$ is the same for augmented random scan and augmented systematic scan, so the same inequality holds for the conductances of the sets.
Finally, because this inequality holds for any set $S$, the inequality also holds for the conductance of the whole chain.
\end{proof}


\levinthm*
\begin{proof}
The lower bound of this inequality,
\begin{align*}
\frac{1/2 - \epsilon}{\Phi_\star}
\leq
t_{\textrm{mix}}(\epsilon)
\end{align*}
is is shown by Theorem 7.3 of \cite{levin2009ams} and holds for any Markov chain -- that is, it does not actually require the Markov chain to be lazy or reversible.

The upper bound of this inequality,
\begin{align*}
t_{\textrm{mix}}(\epsilon)
\leq \frac{2}{\Phi_\star^2}\log\left(\frac{1}{\epsilon\pi_{\textrm{min}}}\right),
\end{align*}
is shown by Theorem 17.10 of \cite{levin2009ams}.
\end{proof}

\thmtmix*
\begin{proof}
\item
\paragraph{Upper Bound for Random Scan:}
First, we upper bound the mixing time of random scan.
\begin{align*}
t_{\textrm{mix}}(R,\epsilon) &\leq \frac{2}{\Phi_{\textrm{RS}}^2}\log\left(\frac{1}{\epsilon\pi_{\textrm{min}}}\right)
\end{align*}

Next, we lower bound the mixing time for systematic scan.
\begin{align*}
t_{\textrm{mix}}(S,\epsilon) &\geq \frac{1/2-\epsilon}{\Phi_{\textrm{SS-A}}}
                              \geq \frac{1/2-\epsilon}{\Phi_{\textrm{RS}}}
\end{align*}

This theorem results from algebraic manipulation of the previous two inequalities.
\begin{align*}
t_{\textrm{mix}}^2(S,\epsilon) &\geq \frac{\left(1/2-\epsilon\right)^2}{\Phi^2_{\textrm{RS}}} \\
\left(1/2-\epsilon\right)^2t_{\textrm{mix}}(R,\epsilon)
&\leq
2t_{\textrm{mix}}^2(S,\epsilon)\log\left(\frac{1}{\epsilon\pi_{\textrm{min}}}\right)
\end{align*}

\paragraph{Upper Bound for Systematic Scan:}
First, we lower bound the mixing time for random scan.
\begin{align*}
t_{\textrm{mix}}(R,\epsilon)   &\geq \frac{1/2 - \epsilon}{\Phi_{\textrm{RS-A}}} \\
t_{\textrm{mix}}(R,\epsilon)^2 &\geq \frac{\left(1/2 - \epsilon\right)^2}{\Phi_{\textrm{RS-A}}^2}
\end{align*}

Next, we manipulate the lower bound of Lemma \ref{thm:condssrs}.
\begin{gather*}
\Phi_{\textrm{SS}} \geq \frac{1}{2n}\cdot\min_{x,i}P_i(x,x)\cdot\Phi_{\textrm{RS-A}} \\
\frac{1}{\Phi_{\textrm{SS}}^2} \leq \frac{4n^2}{\left(\min_{x,i}P_i(x,x)\right)^2}\cdot\frac{1}{\Phi_{\textrm{RS-A}}^2} \\
\end{gather*}

Using this result, we upper bound the mixing time for systematic scan.
\begin{align*}
t_{\textrm{mix}}(S,\epsilon)
\leq \frac{2}{\Phi_{\textrm{SS}}^2}\log\left(\frac{1}{\epsilon\pi_{\textrm{min}}}\right)
\leq \frac{8n^2}{\left(\min_{x,i}P_i(x,x)\right)^2}\cdot\frac{1}{\Phi_{\textrm{RS}}^2}\log\left(\frac{1}{\epsilon\pi_{\textrm{min}}}\right)
\end{align*}

This theorem results from the previous inequalities.
\begin{gather*}
\left(1/2 - \epsilon\right)^2t_{\textrm{mix}}(S,\epsilon)
\leq
\frac{8n^2}{\left(\max_{x,i}P_i(x,x)\right)^2}t_{\textrm{mix}}(R,\epsilon)^2\log\left(\frac{1}{\epsilon\pi_{\textrm{min}}}\right)
\end{gather*}
\end{proof}








}

\end{document}